\documentclass[twoside,11pt]{article}

\usepackage{blindtext}

\usepackage{dsfont}
\usepackage{mathalfa, mathrsfs}
\usepackage{amsfonts, mathtools}
\usepackage[noend]{algorithmic}
\usepackage[ruled,vlined,linesnumbered]{algorithm2e}
\usepackage{tikz, subcaption}
\usetikzlibrary{positioning}
\usepackage[inline]{enumitem}
\usepackage{array, booktabs}

\usetikzlibrary{shapes, arrows, arrows.meta, positioning}
\newcolumntype{M}[1]{>{\centering\arraybackslash}m{#1}}

\newcommand{\independent}{\perp\!\!\!\perp}
\newcommand{\notindependent}{\not\!\perp\!\!\!\perp}
\newcommand{\sep}[5]{ (#1 \independent_{#5} #2 \vert #3)_{#4}}
\newcommand{\nsep}[5]{ (#1 \notindependent_{#5} #2 \vert #3)_{#4}}

\newcommand{\nmarginsep}[4]{(#1 \notindependent_{#4} #2)_{#3}}
\newcommand{\IM}[2]{\text{IM}_{#2}(#1)}
\newcommand{\smax}[1]{\zeta_{\max}(#1)}
\newcommand{\Pa}[2]{\textit{Pa}_{#2}(#1)}
\newcommand{\Ch}[2]{\textit{Ch}_{#2}(#1)}
\newcommand{\Ne}[2]{\textit{Ne}_{#2}(#1)}
\newcommand{\Anc}[2]{\textit{Anc}_{#2}(#1)}
\newcommand{\De}[2]{\textit{De}_{#2}(#1)}
\newcommand{\SCC}[2]{\textit{SCC}_{#2}(#1)}

\newcommand{\V}[0]{\mathbf{V}}

\newcommand{\U}[0]{\mathbf{U}}
\newcommand{\E}[0]{\mathbf{E}}
\newcommand{\X}[0]{\mathbf{X}}

\newcommand{\Y}[0]{\mathbf{Y}}
\newcommand{\Z}[0]{\mathbf{Z}}
\newcommand{\A}[0]{\mathbf{A}}
\newcommand{\B}[0]{\mathbf{B}}

\newcommand{\D}[0]{\mathbf{D}}
\newcommand{\N}[0]{\mathbf{N}}
\newcommand{\G}[0]{\mathcal{G}}
\newcommand{\M}[0]{\mathcal{M}}

\newcommand{\I}[0]{\mathcal{I}}

\usepackage{style}

\ShortHeadings{A Unified Experiment Design Approach for Cyclic and Acyclic Causal Models}{Mokhtarian, Salehkaleybar, Ghassami, and Kiyavash}
\firstpageno{1}

\begin{document}

\title{A Unified Experiment Design Approach for Cyclic and Acyclic Causal Models}

\author{\name Ehsan Mokhtarian \email ehsan.mokhtarian@epfl.ch\\
       \addr School of Computer and Communication Sciences\\
       EPFL\\
       1015 Lausanne, Switzerland\\
       \AND
       \name Saber Salehkaleybar \email s.salehkaleybar@liacs.leidenuniv.nl\\
       \addr Leiden Institute of Advanced Computer Science (LIACS)\\
       Leiden University\\
       2333 CA Leiden, Netherlands\\
       \AND
       \name AmirEmad Ghassami \email ghassami@bu.edu \\
       \addr Department of Mathematics and Statistics,\\
       Boston University\\
       Boston, MA 02215 USA\\
       \AND
       \name Negar Kiyavash \email negar.kiyavash@epfl.ch \\
       \addr College of Management of Technology\\
       EPFL\\
       1015 Lausanne, Switzerland\\
       }

\editor{}

\maketitle

\begin{abstract}%   <- trailing '%' for backward compatibility of .sty file
    We study experiment design for unique identification of the causal graph of a simple SCM, where the graph may contain cycles. The presence of cycles in the structure introduces major challenges for experiment design as, unlike acyclic graphs, learning the skeleton of causal graphs with cycles may not be possible from merely the observational distribution. Furthermore, intervening on a variable in such graphs does not necessarily lead to orienting all the edges incident to it. In this paper, we propose an experiment design approach that can learn both cyclic and acyclic graphs and hence, unifies the task of experiment design for both types of graphs. We provide a lower bound on the number of experiments required to guarantee the unique identification of the causal graph in the worst case, showing that the proposed approach is order-optimal in terms of the number of experiments up to an additive logarithmic term. Moreover, we extend our result to the setting where the size of each experiment is bounded by a constant. For this case, we show that our approach is optimal in terms of the size of the largest experiment required for uniquely identifying the causal graph in the worst case.
\end{abstract}

\section{Introduction}
    One of the fundamental undertakings of empirical sciences is recovering causal relationships among variables of interest in a system \citep{pearl2009causality}.
    Causal relationships are commonly represented by a directed graph (DG), referred to as the causal graph of the system.
    In such a representation, a directed edge from variable $X$ to variable $Y$ denotes that $X$ is a direct cause of $Y$.
    In causal structure learning literature, it is predominantly assumed that the causal graph is a directed acyclic graph (DAG).
    However, in many real-life systems, feedback loops exist among the variables to ensure stability.
    Such feedback loops create cycles in the causal graph of the system when temporal dynamics are sampled at a low rate or when modeling a system's equilibrium states \citep{bongers2021foundations}.
    
    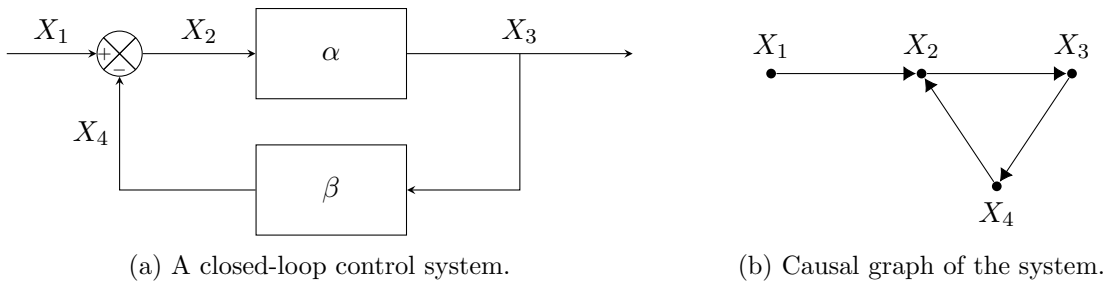
\begin{figure}[b] 
        \centering		
        \begin{subfigure}[b]{0.6\textwidth}
            \centering
            \begin{tikzpicture}
                \node[draw, circle, minimum size=0.6cm] (sum) at (0,0){};
                \draw (sum.north east) -- (sum.south west) 
                    (sum.north west) -- (sum.south east);
             
                \draw (sum.north east) -- (sum.south west)
                    (sum.north west) -- (sum.south east);
                \node[left=-1pt] at (sum.center){\tiny $+$};
                \node[below] at (sum.center){\tiny $-$};

                % Controller
                \node [draw,
                    minimum width=2cm,
                    minimum height=1.2cm,
                    right=1.5cm of sum
                ]  (controller) {$\alpha$};
                
                % feedback block
                \node [draw,
                    minimum width=2cm, 
                    minimum height=1.2cm, 
                    below = 0.6cm of controller
                ]  (feedback) {$\beta$};

                % Arrows with text label
                \draw[-stealth] (sum.east) -- (controller.west)
                    node[midway,above]{$X_2$};
                \draw[-stealth] (controller.east) -- ++ (3,0) 
                    node[midway](output){}node[midway,above]{$X_3$};
                \draw[-stealth] (output.center) |- (feedback.east);
                \draw[-stealth] (feedback.west) -| (sum.south) 
                    node[near end,left]{$X_4$};
                \draw[-stealth] (-1.5,0) -- (sum.west)
                    node[midway,above]{$X_1$};
            \end{tikzpicture}
            \caption{A closed-loop control system.}
            \label{fig: 1a}
        \end{subfigure}
        \hfill
        \begin{subfigure}[b]{0.35\textwidth}
            \centering
            \tikzstyle{block} = [circle, inner sep=1.3pt, fill=black]
            \begin{tikzpicture}
                \tikzset{edge/.style = {->,> = latex',-{Latex[width=2mm]}}}
                % vertices
                \node[block] (X1) at  (0,0) {};
                \node[] ()[above=0 of X1]{$X_1$};
                \node[block] (X2) at  (2,0) {};
                \node[] ()[above=0 of X2]{$X_2$};
                \node[block] (X3) at  (4,0) {};
                \node[] ()[above=0 of X3]{$X_3$};
                \node[block] (X4) at  (3,-1.5) {};
                \node[] ()[below=0 of X4]{$X_4$};

                % %edges
                \draw[edge] (X1) to (X2);
                \draw[edge] (X2) to (X3);
                \draw[edge] (X3) to (X4);
                \draw[edge] (X4) to (X2);
            \end{tikzpicture}
            \caption{Causal graph of the system.}
            \label{fig: 1b}
        \end{subfigure}
        \caption{An example with a feedback loop in control systems that can be modeled with a cyclic SCM (Example \ref{example: simple SCM}).}
        \label{fig: 1}
    \end{figure}
   
    As an example of a system with a feedback loop, consider the closed-loop control system in Figure \ref{fig: 1a} with four variables $X_1,X_2,X_3,X_4$.
    Figure \ref{fig: 1b} illustrates the causal graph among these variables, which is cyclic due to the feedback loop of the control system.
    We shall later revisit this example in more detail in Example \ref{example: simple SCM}.
    Another example appears in gene regulatory networks (GRN), where a collection of biological regulators interact with each other in order to determine the expression level of proteins.
    A GRN can be represented by a DG, where the vertices are the genes, and there is a directed edge from gene $X$ to gene $Y$ if activating gene $X$ may directly activate or suppress gene $Y$.
    Figure \ref{fig: gen} depicts a sub-network of Yeast's GRN \citep{schaffter2011genenetweaver}, where the label of each vertex is the name of the corresponding gene.
    This causal graph is cyclic as it contains a directed cycle (the edges in red).
    Another example is in human circadian rhythms, where genes such as Bmal1, Per2, and Cry1 form feedback loops, creating cycles that control our daily biological clock \citep{pett2018co}.

    \begin{figure}[t] 
        \centering
        \begin{subfigure}[b]{\textwidth}
            \centering
            \tikzstyle{block} = [circle, inner sep=1.3pt, fill=black]
            \begin{tikzpicture}
                \tikzset{edge/.style = {->,> = latex',-{Latex[width=2mm]}}}
                % vertices
                \node[block] (X1) at  (0,0) {};
                \node[] ()[below=0 of X1]{YDL020C};
                \node[block] (X2) at  (2.5,0) {};
                \node[] ()[below=0 of X2]{YBR049C};
                \node[block] (X3) at  (2.5,2) {};
                \node[] ()[above=0 of X3]{YNL216W};
                \node[block] (X4) at  (0,2) {};
                \node[] ()[above=0 of X4]{YKL062W};
                \node[block] (X5) at  (5,2) {};
                \node[] ()[above=0 of X5]{YDR501W};
                \node[block] (X6) at  (7.5,2) {};
                \node[] ()[above=0 of X6]{YIL131C};
                \node[block] (X7) at  (7.5,0) {};
                \node[] ()[below=0 of X7]{YDR451C};
                \node[block] (X8) at  (10,0) {};
                \node[] ()[below=0 of X8]{YKL185W};
                \node[block] (X9) at  (10,2) {};
                \node[] ()[above=0 of X9]{YLR131C};
                \node[block] (X10) at  (5,0) {};
                \node[] ()[below=0 of X10]{YDR216W};

                % %edges
                \draw[edge,red] (X1) to (X2);
                \draw[edge,red] (X2) to (X3);
                \draw[edge,red] (X3) to (X4);
                \draw[edge,red] (X4) to (X1);
                \draw[edge] (X3) to (X5);
                \draw[edge] (X5) to (X6);
                \draw[edge] (X6) to (X9);
                \draw[edge] (X9) to (X8);
                \draw[edge] (X6) to (X7);
                \draw[edge] (X8) to (X7);
                \draw[edge] (X7) to (X10);

            \end{tikzpicture}
        \end{subfigure}
        \caption{A sub-network of Yeast's gene regulatory network that contains a directed cycle of length $4$ (the edges in red).}
        \label{fig: gen}
    \end{figure}
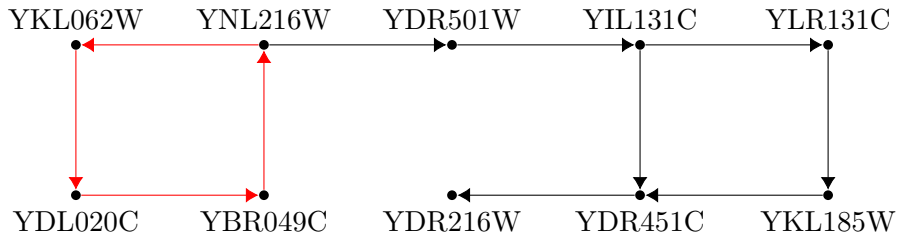
    
    Allowing cycles introduces major challenges to structure learning from observational data.
    For instance, for DAGs, the skeleton of the graph (i.e., the undirected graph obtained by removing the directions of the edges) can be learned from observational data \citep{spirtes2000causation, pearl2009causality, mokhtarian2021recursive, akbari2021recursive}.
    As we shall discuss in Section \ref{subsec: skeleton}, for cyclic DGs, we can only learn a supergraph of the skeleton.
    Another fundamental challenge is as follows.
    For DAGs, if data is generated from a structural causal model (SCM), $d$-Markov property holds, i.e., the joint distribution over the variables contains all the conditional independencies encoded by the $d$-separation relations in the graph.
    For cyclic DGs, this property holds only in specific cases, such as linear systems with continuous variables (see Section \ref{subsec: properties} for a detailed discussion).
    In short, observational data is far less informative for structure learning in the case of cyclic graphs.
    
    To gain more insight into the underlying causal graph, the gold standard is to perform \emph{experiments} in the system.
    That is, to intervene on a subset of variables and study the effect of such intervention on the resulting interventional distribution.
    We refer to the problem of designing a set of experiments sufficient for learning the underlying causal graph as \emph{experiment design problem}.
    As performing experiments are often costly and time-consuming, it is desirable to minimize the number of necessary experiments in the design.

    Experiment design has been studied extensively for DAGs (see related work in Section \ref{sec: related work}).
    Unfortunately, the findings for DAGs are not directly applicable to graphs with cycles.
    For instance, in DAGs, an intervention on a subset of the vertices orients all the edges between the subset and the rest of the variables.
    In Section \ref{subsec: singleton}, we show that in cyclic DGs, performing experiments in some cases neither leads to learning the presence of edges nor orientating them.
    This shows that entirely new techniques are required to develop algorithms for the experiment design problem in the presence of cycles.
   
    To the best of our knowledge, this paper proposes the first unified framework for experiment design for cyclic and acyclic graphs.
    Our main contributions are as follows.
    
    \begin{itemize}[leftmargin=*]
        \item We provide a two-stage experiment design algorithm for learning DG $\G$, the causal graph of a \textit{simple SCM} (Definition \ref{def: simple SCM}).
        In the first stage, we extend the so-called \emph{separating systems} to \emph{colored separating systems} (Definition \ref{def: colored sep sys}), which we utilize to design a set of experiments for learning the strongly connected components (SCC) of $\G$ (Algorithm \ref{algo: SCC}).
        In the second stage, we introduce the novel concept of \emph{lifted separating systems} (Definition \ref{def: lifted sep sys}), which are defined based on the SCCs of the graph.
        As we mentioned before, performing an experiment does not necessarily lead to learning the presence of edges or orientating them.
        However, we show that by performing experiments on the elements of a lifted separating system, we can learn the set of parents of each variable and therefore, exactly recover $\G$ (Algorithm \ref{algo: main}).
        
        \item We provide lower bounds on the number of experiments and the size of the largest experiment that leads to the unique identification of $\G$ for both adaptive and non-adaptive designs in the worst case for any fixed value of $\smax{\G}$, where $\smax{\G}$ denotes the size of the largest SCC of $\G$.
        Specifically, we show that in the worst case, $\G$ cannot be identified by performing experiments with size less than $\smax{\G}-1$ (Theorem \ref{thm: lower bound - size}) or the number of experiments less than $\smax{\G}$ (Theorem \ref{thm: lower bound - number}).
        Additionally, we show that the former bound is tight (Corollary \ref{cor: Thm 1 is tight}), and the latter differs from our achievable bound (Equation \ref{eq: upper bound 1}) by an additive logarithmic term, which demonstrates the order-optimality of our proposed method.
        Note that in acyclic models, the lower bound on the size of the experiments is one since singleton experiments are always sufficient for learning a DAG.
        
        \item Finally, we consider a setup where the size of each designed experiment is bounded by a constant (Section \ref{sec: bounded size}).
        We provide an extension of our approach to this setting and present an upper bound on the number of designed experiments (Equation \ref{eq: upper bound 2}).
        In particular, we formulate the construction of bounded-size lifted separating systems as a combinatorial optimization problem and propose a greedy method for solving it (Theorem \ref{thm: upper bound - bounded size}).
    \end{itemize}
    Table \ref{table: summary} summarizes the main contributions of the paper. 
    The remainder of the paper is organized as follows.
    In Section \ref{sec: Preliminaries}, we review the preliminaries, introduce notations and assumptions, and formally describe the experiment design problem in the presence of cycles.
    In Section \ref{sec: challenges}, we discuss two fundamental challenges of causal discovery from observation or interventional distributions in the presence of cycles.
    In Section \ref{sec: lower bound}, we present lower bounds on the size and number of experiments required for the unique identification of the causal graph.
    In Section \ref{sec: unbounded size}, we propose the two stages of our experiment design algorithm.
    We then generalize our results in Section \ref{sec: bounded size} to the setting where the size of each designed experiment is bounded by a constant.
    In Section \ref{sec: simulation}, we provide a set of simulations over syntactic data sets to illustrate the performance of our method in practice.
    In Section \ref{sec: related work}, we review and discuss related work.
    Finally, in Section \ref{sec: discussion}, we conclude the paper and discuss potential future work.

    \begin{table}[t]
        \fontsize{9}{10.5}\selectfont
        \centering
        \begin{tabular}{M{3.3cm}|M{3.7cm}|M{7cm}}
            \toprule
            & \textbf{Max experiment size}
            & \textbf{Number of experiments}
            \\
            \hline
            \textbf{Unbounded-size alg.}
            & $ n-1$
            & $2\lceil\log_2(\chi(\G_r^{obs}))\rceil + \smax{\G}$
            \\
            % \hline
            \textbf{Bounded-size alg.}
            & $\smax{\G}-1 \leq M < n$
            & $\lceil\frac{n}{M}\rceil \lceil\log_{\lceil\frac{n}{M}\rceil} n\rceil+ \smax{\G}(1+ \lfloor \frac{n-\smax{\G}-1}{M-\smax{\G}+2}\rfloor)$
            \\
            % \hline
            \textbf{Lower bound}
            & $\smax{\G}-1$
            & $\smax{\G}$
            \\
           \bottomrule
        \end{tabular}
        \caption{
        Main contributions of the paper.
        The first two rows provide the achievable bounds on the number of performed experiments for our proposed unbounded-sized (Section \ref{sec: unbounded size}) and bounded-size (Section \ref{sec: bounded size}) experiment design algorithms.
        The last row represents our lower bounds on the number of experiments (Theorem \ref{thm: lower bound - number}) and the size of the largest experiment (Theorem \ref{thm: lower bound - size}) that lead to the unique identification of $\G$ in the worst case.
        The number of variables and the size of the largest SCC of $\G$ are denoted by $n$ and $\smax{\G}$, respectively.
        $\G_r^{obs}$ denotes the skeleton of a graph that can be learned from the observational distribution (Definition \ref{def: G observed}), and $\chi(\G_r^{obs})$ is its coloring number.
        }
        \label{table: summary}
    \end{table}

\section{Preliminaries and Problem Description} \label{sec: Preliminaries}
    Throughout the paper, we denote random variables by capital letters (e.g., $X$), sets of variables by bold letters (e.g., $\X$), and graphs by calligraphic letters (e.g., $\G$).
    \subsection{Preliminary Graph Definitions}
        A \emph{directed graph} (DG) is a graph $\G = (\V,\E)$, where $\V$ is a set of variables and $\E$ is a set of directed edges between the variables in $\V$.
        We denote a directed edge from $X$ to $Y$ by $(X,Y)$, where $X$ is called a \emph{parent} of $Y$ and $Y$ a \emph{child} of $X$.
        Further, \emph{neighbors} of a variable is the union of parents and children of that variable.
        In this paper, we consider DGs without self-loop, i.e., $(X,X)\notin \E$ for all $X\in \V$. 
        However, a DG can have multiple edges (at most one in each direction), i.e., it is possible that $(X,Y)\in \E$ and $(Y,X)\in \E$.
        Similarly, an undirected graph is a graph with undirected edges.
        We denote an undirected edge between two distinct variables $X$ and $Y$ by $\{X,Y\}$.
        The \emph{skeleton} of a DG $\G$ is an undirected graph $(\V, \E')$, where there is an undirected edge $\{X,Y\}$ in $\E'$ if $X$ and $Y$ are neighbors, that is, either $(X,Y)\in \E$ or $(Y,X)\in \E$.
        A \emph{directed acyclic graph} (DAG) is a DG with no cycles.
        
        A \emph{vertex coloring} for an undirected graph $\G$ is an assignment of colors to the vertices, such that no two adjacent vertices are of the same color.
        \emph{Chromatic number} of $\G$, denoted by $\chi(\G)$, is the smallest number of colors needed for a vertex coloring of $\G$.
        
        Suppose $\G =(\V,\E)$ is a DG.
        A path $(X_1,X_2,\cdots,X_k)$ in $\G$ is called a \emph{directed path} from $X_1$ to $X_k$ if $(X_i,X_{i+1})\in \E$ for all $1\leq i<k$.
        Variable $X$ is called an \emph{ancestor} of $Y$ and $Y$ a \emph{descendant} of $X$ if there exists a directed path from $X$ to $Y$ in $\G$.
        Note that $X$ is an ancestor and a descendant of itself.
        A non-endpoint vertex $X$ on a path is called a \emph{collider} if both of the edges incident to $X$ on the path have an arrowhead at $X$.
        A variable $Y$ is \emph{strongly connected} to variable $X$ if $Y$ is both an ancestor and a descendant of $X$.
        We denote the set of parents, children, neighbors, descendants, ancestors, and strongly connected variables of $X$ in $\G$ by $\Pa{X}{\G}$, $\Ch{X}{\G}$, $\Ne{X}{\G}$, $\De{X}{\G}$, $\Anc{X}{\G}$, and $\SCC{X}{\G}$, respectively.
        We will also apply these definitions disjunctively to sets of variables, e.g., $\Pa{\X}{\G} = \bigcup_{X\in \X}\Pa{X}{\G}$ or $\Anc{\X}{\G} = \bigcup_{X\in \X}\Anc{X}{\G}$.
        
        \begin{definition}[SCC]
            Strongly connected variables of $\G$ partition $\V$ into so-called, \emph{strongly connected components} (SCCs); two variables are strongly connected if and only if they are in the same SCC.
            We denote the size of the largest SCC of $\G$ by $\smax{\G}$.
        \end{definition} 
    
    \subsection{Generative Model}
        \emph{Structural causal models} (SCMs) are commonly used to describe the causal mechanisms of a system \citep{pearl2009causality}.
        
        \begin{definition}[SCM]
            An SCM is a tuple $\M = \langle \V,\U,\mathbf{F},P(\U) \rangle$, where $\V$ is a set of endogenous variables, $\U$ is a set of exogenous variables with the joint distribution $P(\U)$ where the variables in $\U$ are assumed to be jointly independent, and $\mathbf{F}$ is a set of functions $\{f_X\}_{X\in \V}$ such that $X = f_X(\Pa{X}{}, \U^X)$, where $\Pa{X}{} \subseteq \V \setminus \{X\}$ and $\U^X\subseteq \U$.
        \end{definition}    
        Let $\M = \langle \V,\U,\mathbf{F},P(\U) \rangle$ be an SCM.
        The assumption of causal sufficiency holds for $\M$ if for any two distinct variables $X,Y\in \V$, $\U^X \cap \U^Y = \varnothing$.
        In this paper, we assume causal sufficiency.
        Under the causal sufficiency assumption, the causal graph of $\M$ is a DG over $\V$ with directed edges from $\Pa{X}{}$ to $X$ for each variable $X\in \V$.

        \begin{definition}[Acyclic SCM]
            An SCM is called acyclic if the corresponding causal graph is a DAG.
            Acyclic SCMs are also known as recursive SEMs.
        \end{definition}
        \begin{remark}
            The definition of SCM does not require the causal graph to be acyclic.
            Acyclicity is often added (or implicitly assumed) in the literature.
        \end{remark}
        Acyclic SCMs have been widely studied in the past few decades because of their convenient properties.
        For instance, they always induce unique observational, interventional, and counterfactual distributions \citep{pearl2009causality}.
        This is not necessarily the case for cyclic SCMs since cycles lead to various complications pertaining to solvability issues.
        \cite{bongers2021foundations} introduced \emph{simple SCMs}, a subclass of SCMs (cyclic or acyclic), which retain most of the convenient properties of acyclic SCMs.

        \begin{definition}[Simple SCM] \label{def: simple SCM}
            An SCM is simple if any subset of its structural equations can be solved uniquely for its associated variables in terms of the other variables that appear in these equations.
        \end{definition}
        We refer the interested reader to \cite{bongers2021foundations} for a more detailed definition of simple SCMs.
        The following result provides a few important properties of simple SCMs.

        \begin{proposition}[\citealt{bongers2021foundations}]
            Simple SCMs always have uniquely defined observational, interventional, and counterfactual distributions.
            \label{prop:simple SCM}
        \end{proposition}
        \begin{example}[Simple SCM] \label{example: simple SCM}
            Consider the control system shown in Figure \ref{fig: 1} with four variables $X_1,X_2,X_3,X_4$.
            The followings are structural equations modeling the control system.
            \begin{equation} \label{eq: control system}
                X_1 = U_1, \quad X_2 = X_1-X_4 + U_2, \quad
                X_3 = \alpha X_2+ U_3, \quad X_4 = \beta X_3 + U_4,
            \end{equation}
            where $\alpha$ and $\beta$ are two constants such that $\alpha \beta \neq -1$, and $U_1, U_2, U_3, U_4$ are independent noise variables ($U_1$ could be viewed as the input to the system and $U_2,U_3,U_4$ as the noise for each state variable in the system).
            This SCM is simple as any subset of the equations in \eqref{eq: control system} can be solved uniquely for its associated variables in terms of the other variables that appear in these equations.
            Proposition \ref{prop:simple SCM} implies that observational, interventional, and counterfactual distributions all exist and are unique for this SCM.
            For instance, suppose we perform an intervention on variable $X_4$ by replacing the corresponding structural equation with $X_4=U_4'$, where $U_4'$ is an independent noise variable.
            This will remove the feedback loop and the variables in the system will be uniquely determined as follows.
            \begin{equation}
                X_1 = U_1, \quad X_2 = U_1- U_4' + U_2, \quad
                X_3 = \alpha (U_1- U_4' + U_2) + U_3, \quad X_4 = U_4'.
            \end{equation}
        \end{example}
    
    \subsection{From \texorpdfstring{$d$}{}-separation to \texorpdfstring{$\sigma$}{}-separation} \label{subsec: properties}
        For three disjoint subsets $\X,\Y,\Z$ of variables with the joint distribution $P$, conditional independence (CI) $\sep{\X}{\Y}{\Z}{P}{}$ denotes that $\X$ and $\Y$ are independent conditioned on $\Z$, i.e., $P(\X,\Y \mid \Z) = P(\X \mid \Z) P(\Y \mid \Z)$.
        %Such a relationship is referred to as conditional independence (CI).

        In the following, we formally define $d$-separation and $\sigma$-separation for DGs.
        
        \begin{definition}[$d$-separation]
            Suppose $\G = (\V,\E)$ is a DG, $X$ and $Y$ are two distinct variables in $\V$, and $\mathbf{S} \subseteq \V \setminus \{X,Y\}$. 
            A path $\mathcal{P}= (X, Z_1, \cdots, Z_k, Y)$ between $X$ and $Y$ in $\G$ is $d$-blocked by $\mathbf{S}$ if there exists $1 \leq i \leq k$ such that
            \begin{itemize}
                \item $Z_i$ is a collider on $\mathcal{P}$ and $Z_i \notin \Anc{\mathbf{S}}{\G}$, or
                \item $Z_i$ is not a collider on $\mathcal{P}$ and $Z_i \in \mathbf{S}$. 
            \end{itemize}
            We say $\mathbf{S}$ $d$-separates $X$ and $Y$ in $\G$ and denote it by $\sep{X}{Y}{\mathbf{S}}{\G}{d}$ if all the paths in $\G$ between $X$ and $Y$ are $d$-blocked by $\mathbf{S}$. 
            For three disjoint subsets $\X,\Y,\mathbf{S}$ in $\V$, we say $\mathbf{S}$ $d$-separates $\X$ and $\Y$ in $\G$, denoted by $\sep{\X}{\Y}{\mathbf{S}}{\G}{d}$, if for any $X \in \X$ and $Y \in \Y$,  $\sep{X}{Y}{\mathbf{S}}{\G}{d}$. 
        \end{definition}
        \begin{definition}[$\sigma$-separation]
            Suppose $\G = (\V,\E)$ is a DG, $X$ and $Y$ are two distinct variables in $\V$, and $\mathbf{S} \subseteq \V \setminus \{X,Y\}$. 
            A path $\mathcal{P}= (X=Z_0, Z_1, \cdots, Z_k, Z_{k+1}=Y)$ between $X$ and $Y$ in $\G$ is $\sigma$-blocked by $\mathbf{S}$ if there exists $1 \leq i \leq k$ such that 
            \begin{itemize}
                \item $Z_i$ is a collider on $\mathcal{P}$ and $Z_i \notin \Anc{\mathbf{S}}{\G}$, or
                \item $Z_i$ is not a collider on $\mathcal{P}$, $Z_i \in \mathbf{S}$, and either $Z_i \to Z_{i+1}$ and $Z_{i+1} \notin \SCC{Z_i}{\G}$, or $Z_{i-1} \gets Z_i$ and $Z_{i-1} \notin \SCC{Z_i}{\G}$. 
            \end{itemize}
            We say $\mathbf{S}$ $\sigma$-separates $X$ and $Y$ in $\G$, denoted by $\sep{X}{Y}{\mathbf{S}}{\G}{\sigma}$, if all the paths in $\G$ between $X$ and $Y$ are $\sigma$-blocked by $\mathbf{S}$. 
            For three disjoint subsets $\X,\Y,\mathbf{S}$ in $\V$, we say $\mathbf{S}$ $\sigma$-separates $\X$ and $\Y$ in $\G$, denoted by $\sep{\X}{\Y}{\mathbf{S}}{\G}{\sigma}$, if for any $X \in \X$ and $Y \in \Y$, $\sep{X}{Y}{\mathbf{S}}{\G}{\sigma}$. 
        \end{definition}
        \begin{remark} \label{rem: d-sigma}
           for DAGs, $\sigma$-separation and $d$-separation are equivalent.
           That is, for three disjoint subsets $\X,\Y,\mathbf{S}$, if the $\sigma$-separation $\sep{\X}{\Y}{\Z}{\G}{\sigma}$ holds, then the $d$-separation $\sep{\X}{\Y}{\Z}{\G}{d}$ holds and visa versa.
           However, for cyclic DGs, the reverse direction does not necessarily hold.
        \end{remark}
        For ease of representation, we introduce letter $r$ to stand for either $d$ (as in $d$-separation) or $\sigma$ (as in $\sigma$-separation).
        Next, we formally define $r$-independence model, $r$-Markov equivalence class, $r$-Markov property, and $r$-faithfulness.
        
        \begin{definition}[$\IM{\G}{r}$]
            For a DG $\G$, the $r$-independence model $\IM{\G}{r}$ is defined as the set of $r$-separations of $\G$.
            That is, 
            \begin{equation*}
                \IM{\G}{r} = \{(X,Y,\Z) \mid X,Y \in \V, \Z \subseteq \V \setminus \{X,Y\},\, \sep{X}{Y}{\Z}{\G}{r}\}.
            \end{equation*}
        \end{definition}
        When $\G$ is a DAG, given their equivalence, we drop subscripts $d$ and $\sigma$ in $d$-separation and $\sigma$-separation notations, respectively, and refer to the independence model as $\IM{\G}{}$ since $\IM{\G}{d} = \IM{\G}{\sigma}$.
        
        \begin{definition}[$r$-MEC]
            Two DGs with identical $r$-independence models are called to be $r$-Markov equivalent.
            We denote by $[\G]^r$ the $r$-Markov equivalence class ($r$-MEC) of $\G$, i.e., the set of $r$-Markov equivalent DGs of $\G$.
        \end{definition}
        \begin{definition}[$r$-Markov property, $r$-faithfulness]
            A distribution $P$ satisfies $r$-Markov property with respect to a DG $\G$ if for any $r$-separation $\sep{\X}{\Y}{\Z}{\G}{r}$ in $\G$, the CI $\sep{\X}{\Y}{\Z}{P}{}$ holds in $P$.
            Similarly, a distribution $P$ satisfies $r$-faithfulness with respect to a DG $\G$ if for any CI $\sep{\X}{\Y}{\Z}{P}{}$ in $P$, the $r$-separation $\sep{\X}{\Y}{\Z}{\G}{r}$ holds in $\G$.
        \end{definition}
        Suppose $\M= \langle \V,\U,\mathbf{F},P(\U) \rangle$ is a simple SCM with observational distribution $P^{\M}(\V)$ and causal graph $\G$. 
        We often drop the superscript $\M$ when it is clear from the context.
        It has been shown that $P$ always satisfies $\sigma$-Markov property with respect to $\G$.
        However, the $d$-Markov property holds in specific settings, e.g., acyclic SCMs, SCMs with continuous variables and linear relations, or SCMs with discrete variables \citep{mooij2020constraint,forre2017markov}.
        On the other hand, $\sigma$-faithfulness is a stronger assumption than $d$-faithfulness due to Remark \ref{rem: d-sigma}.

    \subsection{Intervention and Experiment}
        Suppose $\M= \langle \V,\U,\mathbf{F},P(\U) \rangle$ is an SCM.
        A \emph{full-support hard intervention} on a subset $\mathbf{I} \subseteq \V$, denoted by $do(\mathbf{I})$, converts $\M$ to a new SCM $\M_{do(\mathbf{I})} = \langle \V,\U,\mathbf{F}',P(\U) \rangle$, where for each $X\in \mathbf{I}$, the structural assignment of $X$ in $\mathbf{F}$ is replaced by $X=\xi_X$ in $\mathbf{F}'$, where $\xi_X$ is a random variable whose support is the same as the support of $X$ and is independent of all other random variables in the system.
        We denote the corresponding interventional distribution (i.e., the distribution of $\M_{do(\mathbf{I})}$) by $P_{do(\mathbf{I})}$.
        
        \begin{proposition}[\citealt{bongers2021foundations}]
            If $\M= \langle \V,\U,\mathbf{F},P(\U) \rangle$ is a simple SCM, then for any $\mathbf{I} \subseteq \V$, SCM $\M_{do(\mathbf{I})}$ is also a simple SCM.
        \end{proposition}
        After intervening on $\mathbf{I}$, the variables in $\mathbf{I}$ are no longer functions of other variables in $\V$.
        Hence, the corresponding causal graph of $\M_{do(\mathbf{I})}$ can be obtained from graph $\G$ by removing the incoming edges of the variables in $\mathbf{I}$.
        We denote the resulting graph by $\G_{\overline{\mathbf{I}}}$.
        An \emph{experiment} on a target set $\mathbf{I}$ is the act of conducting a full-support hard intervention on $\mathbf{I}$ and obtaining the interventional distribution $P_{do(\mathbf{I})}$.

        \begin{definition}[$\I$-$r$-MEC]
            Suppose $\I$ is a collection of subsets of $\V$ (can include the empty set).
            Two DGs $\G$ and $\mathcal{H}$ are $\I$-$r$-Markov equivalent if $\IM{\G_{\overline{\mathbf{I}}}}{r} = \IM{\mathcal{H}_{\overline{\mathbf{I}}}}{r}$ for each $\mathbf{I} \in \I$. 
            We denote by $[\G]^r_{\I}$ the $\I$-$r$-Markov equivalent class of $\G$, i.e., the set of $\I$-$r$-Markov equivalent DGs of $\G$.
        \end{definition}
        This definition implies that it is impossible to distinguish two $\I$-$r$-Markov equivalent graphs by the $r$-separations of the graphs resulting from experiments on the elements of $\I$.

    \subsection{Problem Description}\label{sec:desc}
        Consider a simple SCM $\M= \langle \V,\U,\mathbf{F},P(\U) \rangle$  with observational distribution $P^{\M}(\V)$ and causal graph $\G$.
        We assume causal sufficiency, in which case $\G$ is a DG.\footnote{Note that DGs cannot represent the presence of hidden confounders.
        Without causal sufficiency, the causal graph can be represented by a directed mixed graph (DMG) that contains bidirected edges to indicate the presence of hidden confounders.}
        As discussed earlier, in simple SCMs, $\sigma$-Markov property always holds (even in non-linear systems with continuous variables), while $d$-Markov property holds in certain settings. 
        On the other hand, $\sigma$-faithfulness is a stronger assumption than $d$-faithfulness.
        In this paper, we consider the following two scenarios.
        \begin{itemize}[leftmargin=*]
            \item \textbf{Scenario 1:} $P^{\M}$ satisfies $d$-Markov property and $d$-faithfulness w.r.t. $\G$.
            In this case, CI relations are equivalent to $d$-separations. 
            That is, $\sep{X}{Y}{\Z}{\G}{d} \iff \sep{X}{Y}{\Z}{P}{}$.
            \item \textbf{Scenario 2:} $P^{\M}$ satisfies $\sigma$-faithfulness w.r.t. $\G$.
            In this case, CI relations are equivalent to $\sigma$-separations.
            That is, $\sep{X}{Y}{\Z}{\G}{\sigma} \iff \sep{X}{Y}{\Z}{P}{}$.
        \end{itemize}
        Note that if $\G$ is a DAG, the aforementioned scenarios are the same.
        However, if there are cycles in $\G$, the two scenarios are not necessarily equivalent.
        
        Our goal in this paper is to design a set of experiments for learning $\G$ under Scenario $1$ or Scenario $2$. 
        That is, to introduce a collection of subsets $\I$ such that $[\G]^r_{\I} = \{\G\}$.
        Additionally, as performing experiments can be costly, we aim to minimize the number of necessary experiments.
        Surprisingly, our approaches for both scenarios coincide, but the proof techniques differ.

\section{Challenges of Experiment Design in Presence of Cycles} \label{sec: challenges}
    In this section, we discuss some of the challenges pertaining to learning DGs in the presence of cycles.
    In Subsection \ref{subsec: skeleton}, we show that, unlike DAGs, we cannot learn the skeleton of a DG without performing experiments, i.e., from merely the observational distribution.
    In Subsection \ref{subsec: singleton}, we argue that even performing all size-one experiments (singleton experiments) does not suffice to learn $\G$ in some cases.
    
    \subsection{Skeleton of a DG is not Learnable from Observational Distribution} \label{subsec: skeleton}
        For any DAG $\G = (\V,\E)$, \cite{verma1990equivalence} showed that non-neighbor variables are $d$-separable.
        That is, for any distinct and non-neighbor variables $X$ and $Y$, there exists a subset of $\V \setminus \{X,Y\}$ that $d$-separates $X$ and $Y$.
        This implies that the observational distribution suffices to learn the skeleton of $\G$.
        In the following, we show that this assertion is not true in cyclic graphs for either of the two scenarios introduced in \ref{sec:desc}. 
        Let us begin by defining the skeleton of a graph that can be learned from the observational distribution.
        
        \begin{definition}[$\G_r^{obs}$] \label{def: G observed}
            Suppose $\G=(\V,\E)$ is a DG. 
            Let $\G_r^{obs}$ denote the undirected graph over $\V$ where there is an edge between $X$ and $Y$ if and only if $X$ and $Y$ are not $r$-separable in $\G$, i.e., for any $\mathbf{S}\subseteq \V \setminus \{X,Y\}$ we have $\nsep{X}{Y}{\mathbf{S}}{\G}{r}$.
        \end{definition}
        Note that $\G_r^{obs}$ includes the skeleton of $\G$ but can potentially have additional edges.  
        Next, we describe $\G_d^{obs}$ in Scenario 1 and $\G_{\sigma}^{obs}$ in Scenario 2.
                
        \subsubsection{Scenario 1}
        \begin{example}[Virtual edge] \label{example: virtual edge}
            Consider DG $\G$ in Figure \ref{fig: 2a}. 
            In this graph, $Y$ and $X_4$ are not $d$-separable.
            Thus, there can be an edge between $Y$ and $X_4$ in some of the DGs in $[\G]^d$, such as in DG $\G_1$ in Figure \ref{fig: 2b}.
        \end{example}
        In Example \ref{example: virtual edge}, a so-called \emph{virtual edge} exists between $Y$ and $X_4$ which we formally define in the following \citep{richardson1996polynomial, ghassami2020characterizing}.
        
        \begin{definition}[Virtual edge]
            There exists a virtual edge between two non-neighbor variables $X$ and $Z$ in a DG $\G = (\V,\E)$ if $X$ and $Z$ have a common child that is an ancestor of either $X$ or $Z$, i.e., $\Ch{X}{\G} \cap \Ch{Z}{\G}\cap \Anc{\{X,Z\}}{\G} \neq \varnothing$.
        \end{definition}
        The following result demonstrates the importance of virtual edges.
        
        \begin{proposition}[\citealt{richardson1996polynomial}] \label{prp: virtual edge}
            Two variables are $d$-separable in DG $\G$ if and only if an edge or virtual edge does not connect them.
            Accordingly, $\G_d^{obs}$ is obtained by adding the virtual edges of $\G$ to the skeleton of $\G$.
        \end{proposition}
        For DG $\G$ in Figure \ref{fig: 2a} (Example \ref{example: virtual edge}), there exists a virtual edge between $Y$ and $X_4$ because $X_1 \in \Ch{Y}{\G} \cap \Ch{X_4}{\G}\cap \Anc{\{Y,X_4\}}{\G}$.
        Figure \ref{fig: 2c} depicts $\G_d^{obs}$.

        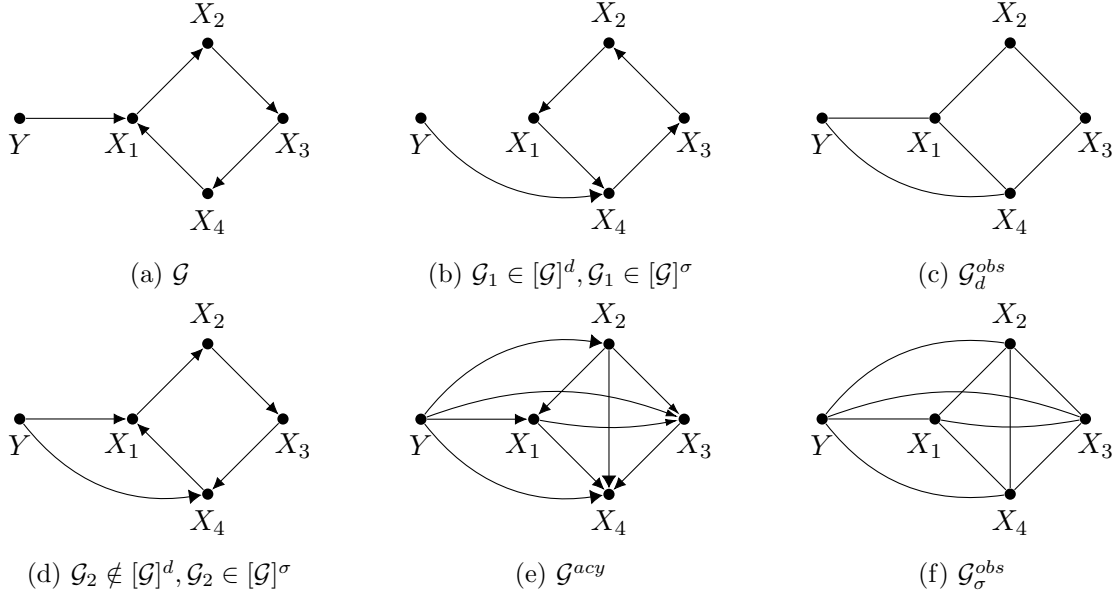
\begin{figure}[t] 
            \centering
            \tikzstyle{block} = [circle, inner sep=1.5pt, fill=black]
            \tikzstyle{input} = [coordinate]
            \tikzstyle{output} = [coordinate]
            \begin{subfigure}[b]{0.3\textwidth}
        	\centering
                \begin{tikzpicture}
                    \tikzset{edge/.style = {->,> = latex',-{Latex[width=1.5mm]}}}
                    % vertices
                    \node[block] (Y) at  (0,0) {};
                    \node[] ()[below=0 of Y]{$Y$};
                    \node[block] (X1) at  (1.5,0) {};
                    \node[] ()[below left =0 and -0.3 of X1]{$X_1$};
                    \node[block] (X2) at  (2.5,1) {};
                    \node[] ()[above=0 of X2]{$X_2$};
                    \node[block] (X3) at  (3.5,0) {};
                    \node[] ()[below right=0 and -0.3 of X3]{$X_3$};
                    \node[block] (X4) at  (2.5,-1) {};
                    \node[] ()[below=0 of X4]{$X_4$};
                    %edges
                    \draw[edge] (Y) to (X1);
                    \draw[edge] (X1) to (X2);
                    \draw[edge] (X2) to (X3);
                    \draw[edge] (X3) to (X4);
                    \draw[edge] (X4) to (X1);
                \end{tikzpicture}
                \caption{$\G$}
                \label{fig: 2a}
            \end{subfigure}
            \hfill
            \begin{subfigure}[b]{0.3\textwidth}
        		\centering
        		\begin{tikzpicture}
                    \tikzset{edge/.style = {->,> = latex',-{Latex[width=1.5mm]}}}
                    % vertices
                    \node[block] (Y) at  (0,0) {};
                    \node[] ()[below=0 of Y]{$Y$};
                    \node[block] (X1) at  (1.5,0) {};
                    \node[] ()[below left =0 and -0.3 of X1]{$X_1$};
                    \node[block] (X2) at  (2.5,1) {};
                    \node[] ()[above=0 of X2]{$X_2$};
                    \node[block] (X3) at  (3.5,0) {};
                    \node[] ()[below right=0 and -0.3 of X3]{$X_3$};
                    \node[block] (X4) at  (2.5,-1) {};
                    \node[] ()[below=0 of X4]{$X_4$};
                    %edges
                    \draw[edge] (X2) to (X1);
                    \draw[edge] (X3) to (X2);
                    \draw[edge] (X4) to (X3);
                    \draw[edge] (X1) to (X4);
                    \path[->] (Y) edge[style = {->,> = latex',-{Latex[width=2mm]}}, bend right=30](X4);
                \end{tikzpicture}
                \caption{$\G_1 \in [\G]^d, \G_1 \in [\G]^{\sigma}$}
                \label{fig: 2b}
            \end{subfigure}
            \hfill
            \begin{subfigure}[b]{0.3\textwidth}
        		\centering
        		\begin{tikzpicture}
                    % vertices
                    \node[block] (Y) at  (0,0) {};
                    \node[] ()[below=0 of Y]{$Y$};
                    \node[block] (X1) at  (1.5,0) {};
                    \node[] ()[below left =0 and -0.3 of X1]{$X_1$};
                    \node[block] (X2) at  (2.5,1) {};
                    \node[] ()[above=0 of X2]{$X_2$};
                    \node[block] (X3) at  (3.5,0) {};
                    \node[] ()[below right=0 and -0.3 of X3]{$X_3$};
                    \node[block] (X4) at  (2.5,-1) {};
                    \node[] ()[below=0 of X4]{$X_4$};
                    %edges
                    \draw[-] (Y) to (X1);
                    \draw[-] (X1) to (X2);
                    \draw[-] (X2) to (X3);
                    \draw[-] (X3) to (X4);
                    \draw[-] (X4) to (X1);
                    \path[-] (Y) edge[style = {-}, bend right=30](X4);
                \end{tikzpicture}
                \caption{$\G_d^{obs}$}
                \label{fig: 2c}
            \end{subfigure}
            
            \begin{subfigure}[b]{0.3\textwidth}
        		\centering
        		\begin{tikzpicture}
                    \tikzset{edge/.style = {->,> = latex',-{Latex[width=1.5mm]}}}
                    % vertices
                    \node[block] (Y) at  (0,0) {};
                    \node[] ()[below=0 of Y]{$Y$};
                    \node[block] (X1) at  (1.5,0) {};
                    \node[] ()[below left =0 and -0.3 of X1]{$X_1$};
                    \node[block] (X2) at  (2.5,1) {};
                    \node[] ()[above=0 of X2]{$X_2$};
                    \node[block] (X3) at  (3.5,0) {};
                    \node[] ()[below right=0 and -0.3 of X3]{$X_3$};
                    \node[block] (X4) at  (2.5,-1) {};
                    \node[] ()[below=0 of X4]{$X_4$};
                    %edges
                    \draw[edge] (Y) to (X1);
                    \draw[edge] (X1) to (X2);
                    \draw[edge] (X2) to (X3);
                    \draw[edge] (X3) to (X4);
                    \draw[edge] (X4) to (X1);
                    \path[->] (Y) edge[style = {->,> = latex',-{Latex[width=2mm]}}, bend right=30](X4);
                \end{tikzpicture}
                \caption{$\G_2 \notin [\G]^d, \G_2 \in [\G]^{\sigma}$}
                \label{fig: 2d}
            \end{subfigure}
            \hfill
            \begin{subfigure}[b]{0.3\textwidth}
        	\centering
                \begin{tikzpicture}
                    \tikzset{edge/.style = {->,> = latex',-{Latex[width=1.5mm]}}}
                    \tikzset{edge1/.style = {->,> = latex',-{Latex[width=1mm]}}}
                    % vertices
                    \node[block] (Y) at  (0,0) {};
                    \node[] ()[below=0 of Y]{$Y$};
                    \node[block] (X1) at  (1.5,0) {};
                    \node[] ()[below left =0 and -0.3 of X1]{$X_1$};
                    \node[block] (X2) at  (2.5,1) {};
                    \node[] ()[above=0 of X2]{$X_2$};
                    \node[block] (X3) at  (3.5,0) {};
                    \node[] ()[below right=0 and -0.3 of X3]{$X_3$};
                    \node[block] (X4) at  (2.5,-1) {};
                    \node[] ()[below=0 of X4]{$X_4$};
                    %edges
                    \draw[edge] (Y) to (X1);
                    \draw[edge] (X2) to (X1);
                    \draw[edge1] (X2) to (X3);
                    \draw[edge] (X3) to (X4);
                    \draw[edge] (X1) to (X4);
                    \path[->] (Y) edge[style = {->,> = latex',-{Latex[width=2mm]}}, bend left=30](X2);
                    \path[->] (Y) edge[style = {->,> = latex',-{Latex[width=2mm]}}, bend right=30](X4);
                    \path[->] (Y) edge[style = {->,> = latex',-{Latex[width=1mm]}}, bend left=20](X3);
                    \path[->] (X1) edge[style = {->,> = latex',-{Latex[width=1mm]}}, bend right=10](X3);
                    \path[->] (X2) edge[style = {->,> = latex',-{Latex[width=2mm]}}](X4);
                \end{tikzpicture}
                \caption{$\G^{acy}$}
                \label{fig: 2e}
            \end{subfigure}
            \hfill
            \begin{subfigure}[b]{0.3\textwidth}
        	\centering
                \begin{tikzpicture}
                    % vertices
                    \node[block] (Y) at  (0,0) {};
                    \node[] ()[below=0 of Y]{$Y$};
                    \node[block] (X1) at  (1.5,0) {};
                    \node[] ()[below left =0 and -0.3 of X1]{$X_1$};
                    \node[block] (X2) at  (2.5,1) {};
                    \node[] ()[above=0 of X2]{$X_2$};
                    \node[block] (X3) at  (3.5,0) {};
                    \node[] ()[below right=0 and -0.3 of X3]{$X_3$};
                    \node[block] (X4) at  (2.5,-1) {};
                    \node[] ()[below=0 of X4]{$X_4$};
                    %edges
                    \draw[-] (Y) to (X1);
                    \draw[-] (X2) to (X1);
                    \draw[-] (X2) to (X3);
                    \draw[-] (X3) to (X4);
                    \draw[-] (X1) to (X4);
                    \path[-] (Y) edge[style = {-}, bend left=30](X2);
                    \path[-] (Y) edge[style = {-}, bend right=30](X4);
                    \path[-] (Y) edge[style = {-}, bend left=20](X3);
                    \path[-] (X1) edge[style = {-}, bend right=10](X3);
                    \path[-] (X2) edge[style = {-}](X4);
                \end{tikzpicture}
                \caption{$\G_{\sigma}^{obs}$}
                \label{fig: 2f}
            \end{subfigure}
            \caption{
            Figures \ref{fig: 2a}, \ref{fig: 2b}, and \ref{fig: 2c} depict a cyclic DG $\G$, a cyclic DG in $[\G]^d$, and undirected graph $\G_d^{obs}$, respectively (Example \ref{example: virtual edge}). 
            The DG in Figure \ref{fig: 2d} belongs to $[\G]^{\sigma}$ but does not belong to $[\G]^d$.
            Figures \ref{fig: 2e} and \ref{fig: 2f} depict a $\sigma$-acyclification of $\G$ and $\G_{\sigma}^{obs}$, respectively (Example \ref{example: acyclification}).}
            \label{fig: 2}
        \end{figure}
        
        \subsubsection{Scenario 2}
            \cite{mooij2020constraint} introduced the notion of $\sigma$-acyclification as follows.
            
            \begin{definition}[$\sigma$-acyclification] \label{def: acyclification}
                Suppose $\G=(\V,\E)$ is a DG. 
                A $\sigma$-acyclification of $\G$ is a DAG $\G'=(\V,\E')$ that satisfies the followings.
                \begin{enumerate}
                    \item For any $X\in \V$ and $Y \in \V \setminus \SCC{X}{\G}$, $(X,Y) \in \E'$ if and only if there exists $Z \in \SCC{Y}{\G}$ such that $(X,Z)\in \E$.
                    \item For any $X\in \V$ and $Y \in \SCC{X}{\G} \setminus \{X\}$, either $(X,Y) \in \E'$ or $(Y,X)\in \E'$.
                \end{enumerate}
            \end{definition}
            Note that $\sigma$-acyclification is not unique since Definition \ref{def: acyclification} does not uniquely orient the edges between two variables in the same SCC.
            \begin{proposition}[\citealt{mooij2020constraint}] \label{prp: acyclification}
                There exists at least one $\sigma$-acyclification of any DG.
                Furthermore, if DAG $\G^{acy}$ is a $\sigma$-acyclification of a DG $\G$, then
                \begin{equation*}
                    \IM{\G}{\sigma} = \IM{\G^{acy}}{}.
                \end{equation*}
            \end{proposition}
            Suppose DAG $\G^{acy}$ is a $\sigma$-acyclification of a DG $\G$.
            Proposition \ref{prp: acyclification} implies that $\G_{\sigma}^{obs}$ is the skeleton of $\G^{acy}$.
            Furthermore, the following corollary pertaining to the skeleton of $\G^{acy}$ follows from the definition of $\sigma$-acyclification.
            \begin{corollary}
                There exists an edge between two distinct variables $X$ and $Y$ in $\G_{\sigma}^{obs}$ if and only if $Y \in \SCC{X}{\G}$ or there exists $Z \in \SCC{X}{\G}$ such that $Y\in \Pa{Z}{\G}$. 
            \end{corollary}
            This corollary describes how $\G_{\sigma}^{obs}$ is obtained from $\G$.
            Note that the skeleton of any DG in $[\G]^{\sigma}$ is a subgraph of $\G_{\sigma}^{obs}$.

            \begin{example} \label{example: acyclification}
                Consider again the example in Figure \ref{fig: 2}.
                It can be shown that $\G$, $\G_1$, and $\G_2$ (Figure \ref{fig: 2d}) do not induce any $\sigma$-separation, i.e., $\IM{\G}{\sigma} = \IM{\G_1}{\sigma}= \IM{\G_2}{\sigma}= \varnothing$ and therefore, $\G_1,\G_2 \in [\G]^{\sigma}$.
                Note that $\G_2 \notin [\G]^d$ because $\sep{Y}{X_2}{X_1,X_4}{\G}{d}$ but $\nsep{Y}{X_2}{X_1,X_4}{\G_2}{d}$.
                Surprisingly, we can construct many other DGs (more than 1000 DGs) that are in $[\G]^\sigma$ but are not in $[\G]^d$.
                Figures \ref{fig: 2e} and \ref{fig: 2f} depict a $\sigma$-acyclification of $\G$ and $\G_{\sigma}^{obs}$, respectively.
            \end{example}  
            To sum up this section, observational distribution does not suffice to distinguish between (i)  actual edges and virtual edges in Scenario 1 and (ii) actual edges, virtual edges, and the additional edges of $\G_{\sigma}^{obs}$ in Scenario 2. 
            Furthermore, $[\G]^{d}$ or $[\G]^{\sigma}$ can contain a large number of graphs with various skeletons, and it is necessary to perform experiments in order to learn the skeleton of $\G$.
        
        \subsection{Singleton Experiments are not Sufficient} \label{subsec: singleton}
                        
            A \emph{singleton experiment} refers to an experiment in which the target set is comprised of a single variable. 
            In  DAGs, the children of a variable could be identified by performing a singleton experiment on it. 
            Hence, the whole graph can be learned by performing singleton experiments on all the variables.
            Herein, we show that this does not hold for cyclic DGs.

            \begin{example} \label{example: singleton}
                Consider the DGs in Figure \ref{fig: 3} and the set of singleton experiments (including the empty set) $\I =\{\varnothing, \{X\}, \{Y\}, \{Z\}\}$.
                For any DG $\G$ in this figure and any experiment $\mathbf{I} \in \I$, $\IM{\G_{\overline{\mathbf{I}}}}{r} = \varnothing$.
                Hence, all of the DGs in Figure \ref{fig: 3} are $\I$-$r$-Markov equivalent.
                It is noteworthy that removing \emph{any} edge in the left DG in Figure \ref{fig: 3} results in a DG in the same $\I$-$r$-Markov equivalent class.
            \end{example}
            This example illustrates that a DG cannot always be learned through singleton experiments, even if they are performed on all variables.

            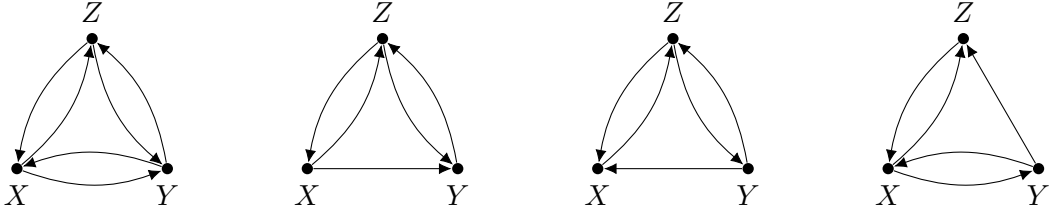
\begin{figure}[t] 
        	\centering
        	\tikzstyle{block} = [circle, inner sep=1.5pt, fill=black]
        	\tikzstyle{input} = [coordinate]
        	\tikzstyle{output} = [coordinate]
                \begin{subfigure}[b]{0.24\textwidth}
            	\centering
                    \begin{tikzpicture}
                        \tikzset{edge/.style = {->,> = latex',-{Latex[width=1.5mm]}}}
                        % vertices
                        \node[block] (X) at  (0,0) {};
                        \node[] ()[below=0 of X]{$X$};
                        \node[block] (Y) at  (2,0) {};
                        \node[] ()[below=0 of Y]{$Y$};
                        \node[block] (Z) at  (1,1.73) {};
                        \node[] ()[above=0 of Z]{$Z$};
                        %edges
                        \path[->] (X) edge[style = {->,> = latex',-{Latex[width=1.5mm]}}, bend right=20](Z);
                        \path[->] (Z) edge[style = {->,> = latex',-{Latex[width=1.5mm]}}, bend right=20](X);
                        \path[->] (X) edge[style = {->,> = latex',-{Latex[width=1.5mm]}}, bend right=20](Y);
                        \path[->] (Y) edge[style = {->,> = latex',-{Latex[width=1.5mm]}}, bend right=20](X);
                        \path[->] (Z) edge[style = {->,> = latex',-{Latex[width=1.5mm]}}, bend right=20](Y);
                        \path[->] (Y) edge[style = {->,> = latex',-{Latex[width=1.5mm]}}, bend right=20](Z);
                    \end{tikzpicture}
                \end{subfigure}
                \hfill
                \begin{subfigure}[b]{0.24\textwidth}
                    \centering
                    \begin{tikzpicture}
                        \tikzset{edge/.style = {->,> = latex',-{Latex[width=1.5mm]}}}
                        % vertices
                        \node[block] (X) at  (0,0) {};
                        \node[] ()[below=0 of X]{$X$};
                        \node[block] (Y) at  (2,0) {};
                        \node[] ()[below=0 of Y]{$Y$};
                        \node[block] (Z) at  (1,1.73) {};
                        \node[] ()[above=0 of Z]{$Z$};
                        %edges
                        \path[->] (X) edge[style = {->,> = latex',-{Latex[width=1.5mm]}}, bend right=20](Z);
                        \path[->] (Z) edge[style = {->,> = latex',-{Latex[width=1.5mm]}}, bend right=20](X);
                        \draw[edge] (X) to (Y);
                        \path[->] (Z) edge[style = {->,> = latex',-{Latex[width=1.5mm]}}, bend right=20](Y);
                        \path[->] (Y) edge[style = {->,> = latex',-{Latex[width=1.5mm]}}, bend right=20](Z);
                    \end{tikzpicture}
                \end{subfigure}
                \hfill
                \begin{subfigure}[b]{0.24\textwidth}
                    \centering
                    \begin{tikzpicture}
                        \tikzset{edge/.style = {->,> = latex',-{Latex[width=1.5mm]}}}
                        % vertices
                        \node[block] (X) at  (0,0) {};
                        \node[] ()[below=0 of X]{$X$};
                        \node[block] (Y) at  (2,0) {};
                        \node[] ()[below=0 of Y]{$Y$};
                        \node[block] (Z) at  (1,1.73) {};
                        \node[] ()[above=0 of Z]{$Z$};
                        %edges
                        \path[->] (X) edge[style = {->,> = latex',-{Latex[width=1.5mm]}}, bend right=20](Z);
                        \path[->] (Z) edge[style = {->,> = latex',-{Latex[width=1.5mm]}}, bend right=20](X);
                        \draw[edge] (Y) to (X);
                        \path[->] (Z) edge[style = {->,> = latex',-{Latex[width=1.5mm]}}, bend right=20](Y);
                        \path[->] (Y) edge[style = {->,> = latex',-{Latex[width=1.5mm]}}, bend right=20](Z);
                    \end{tikzpicture}
                \end{subfigure}
                \hfill
                \begin{subfigure}[b]{0.24\textwidth}
            	\centering
                    \begin{tikzpicture}
                        \tikzset{edge/.style = {->,> = latex',-{Latex[width=1.5mm]}}}
                        % vertices
                        \node[block] (X) at  (0,0) {};
                        \node[] ()[below=0 of X]{$X$};
                        \node[block] (Y) at  (2,0) {};
                        \node[] ()[below=0 of Y]{$Y$};
                        \node[block] (Z) at  (1,1.73) {};
                        \node[] ()[above=0 of Z]{$Z$};
                        %edges
                        \path[->] (X) edge[style = {->,> = latex',-{Latex[width=1.5mm]}}, bend right=20](Z);
                        \path[->] (Z) edge[style = {->,> = latex',-{Latex[width=1.5mm]}}, bend right=20](X);
                        \path[->] (X) edge[style = {->,> = latex',-{Latex[width=1.5mm]}}, bend right=20](Y);
                        \path[->] (Y) edge[style = {->,> = latex',-{Latex[width=1.5mm]}}, bend right=20](X);
                        \draw[edge] (Y) to (Z);
                    \end{tikzpicture}
                \end{subfigure}
                \caption{Four $\I$-$r$-Markov equivalent DGs, where $\I=\{\varnothing, \{X\}, \{Y\}, \{Z\}\}$ (Example \ref{example: singleton}).}
                \label{fig: 3}
            \end{figure}

\section{Lower Bounds on Number and Size of Experiment Sets} \label{sec: lower bound}
    As we discussed, performing singleton experiments does not suffice for learning a DG in some cases.
    In this section, we provide lower bounds on both the number and size of experiments required to learn a DG in the worst case.
    For any constant $c<n$, we show that among the DGs with maximum SCC size of $c$, there exists a DG $\G$ that is not uniquely identifiable by performing experiments with size less than $\smax{\G}-1$ or conducting less than $\smax{\G}$ experiments, where $\smax{\G}$ denotes the size of the largest SCC of $\G$.
    To this end, we first provide the following lemma.

    \begin{lemma} \label{lemma: fundamental}
        Consider a set of $n$ vertices denoted by $\V$ and a constant $1< c \leq n$.
        For an arbitrary subset $\V_c \subseteq \V$ with $|\V_c|=c$, let $\G = (\V,\E)$, where $\E = \{(X,Y) \mid X,Y \in \V_c, X\neq Y\}$.
        For two arbitrary and distinct variables $X^*$ and $Y^*$ in $\V_c$, let DG $\G'=(\V, \E \setminus \{(X^*,Y^*)\})$.
        For a set of experiments $\I$ on $\V$, $\G' \notin [\G]^r_{\I}$ if and only if 
        \begin{equation} \label{eq: fundamental}
            \exists \mathbf{I} \in \I: \quad \mathbf{I}\cap \V_c = \V_c \setminus \{Y^*\}.
        \end{equation}
    \end{lemma}
    \begin{proof}
        \textbf{Sufficiency:}
        If there exists a subset $\mathbf{I} \in \I$ such that $\mathbf{I}\cap \V_c = \V_c \setminus \{Y^*\}$, after intervening on $\mathbf{I}$, $X^*$ and $Y^*$ are $r$-separable in $\G'$.
        Hence, $\G' \notin [\G]^r_{\I}$.

        \textbf{Necessity:}
        Suppose Equation \eqref{eq: fundamental} does not hold.
        We need to show that $\G' \in [\G]^r_{\I}$.
        That is, for any $\mathbf{I} \in \I$ ($\mathbf{I}$ can be the empty set), we need to show that $\IM{\G_{\overline{\mathbf{I}}}}{r} = \IM{\G'_{\overline{\mathbf{I}}}}{r}$.
        Note that $\G'_{\overline{\mathbf{I}}} \subseteq \G_{\overline{\mathbf{I}}}$ since $\G' \subseteq \G$.
        Hence, $\IM{\G_{\overline{\mathbf{I}}}}{r} \subseteq \IM{\G'_{\overline{\mathbf{I}}}}{r}$.
        Let $(X,Y,\mathbf{S}) \in \IM{\G'_{\overline{\mathbf{I}}}}{r}$. 
        To complete the proof, we will show that $(X,Y,\mathbf{S}) \in \IM{\G_{\overline{\mathbf{I}}}}{r}$.
        
        If $X\in \V \setminus \V_c$ or $Y \in \V \setminus \V_c$, then $(X,Y,\mathbf{S})\in \IM{\G_{\overline{\mathbf{I}}}}{r}$ because the variables in $\V \setminus \V_c$ do not have any neighbors.
        Now, suppose $X,Y \in \V_c$.
        Since the variables in $\V_c \setminus \mathbf{I}$ are neighbors in $\G'$ (note that $X^*$ and $Y^*$ are neighbors because $(Y^*,X^*) \in \E$), at least one of $X$ or $Y$ is in $\mathbf{I}$.
        Without loss of generality, let us assume that $X \in \mathbf{I}$.
        
        Next, we show that $Y$ is also in $\mathbf{I}$.
        Assume by contradiction that $Y \in \V_c \setminus \mathbf{I}$.
        Since $\mathbf{S}$ $r$-separates $X$ and $Y$ in $\G'_{\overline{\mathbf{I}}}$, $(X,Y) \notin \G'$ which implies that $X= X^*$ and $Y=Y^*$.
        In this case, Equation \eqref{eq: fundamental} implies that $\V_c \setminus (\mathbf{I} \cup \{Y\})$ is non-empty.
        Let $Z$ be a variable in $\V_c \setminus (\mathbf{I} \cup \{Y\})$.
        This implies that $Z \in \Ch{X}{\G'_{\overline{\mathbf{I}}}} \cap \Ch{Y}{\G'_{\overline{\mathbf{I}}}} \cap \Pa{Y}{\G'_{\overline{\mathbf{I}}}}$.
        Hence, there is a virtual edge between $X$ and $Y$ in $\G'_{\overline{\mathbf{I}}}$ and therefore, they are not $r$-separable which is a contradiction.
        This implies that $Y$ is in $\mathbf{I}$.
        
        So far we have shown that $X,Y \in \V_c \cap \mathbf{I}$.
        Due to the structure of $\G$, $(X,Y,\mathbf{S})\in \IM{\G_{\overline{\mathbf{I}}}}{r}$ if and only if $\mathbf{S} \cap (\V_c \setminus \mathbf{I}) = \varnothing$.
        Accordingly, to complete the proof, it suffices to show that $\mathbf{S} \cap (\V_c \setminus \mathbf{I}) = \varnothing$.
        Assume by contradiction that there exists a variable $Z_1$ in $ \mathbf{S} \cap (\V_c \setminus \mathbf{I})$.
        In this case, $Z_1=Y^*$ and $X^* \in \{X,Y\}$ because otherwise, $\mathbf{S}$ does not $r$-block path $X \to Z_1 \gets Y$ in $\G'_{\overline{\mathbf{I}}}$. 
        Without loss of generality suppose $X^* = X$.
        Again, Equation \eqref{eq: fundamental} implies that there exists a variable $Z_2$ in $\V_c \setminus (\mathbf{I} \cup \{Z_1\})$.
        Since $\mathbf{S}$ $r$-blocks path $X \to Z_2 \gets Y$ in $\G'_{\overline{\mathbf{I}}}$, then $Z_2 \notin \mathbf{S}$.
        In this case, $\mathbf{S}$ does not $r$-block path $X \to Z_2 \to Z_1 \gets Y$ in $\G'_{\overline{\mathbf{I}}}$ which is a contradiction.
        This shows that $\mathbf{S} \cap (\V_c \setminus \mathbf{I}) = \varnothing$ and therefore, $(X,Y,\mathbf{S})\in \IM{\G_{\overline{\mathbf{I}}}}{r}$, which completes the proof.
    \end{proof}

    We now present two important consequences of Lemma \ref{lemma: fundamental}.
    
    \begin{theorem} \label{thm: lower bound - size}
        Consider a set of $n$ vertices denoted by $\V$ and a constant $1< c \leq n$.
        There exists a DG $\G$ over $\V$ with $\smax{\G}=c$ such that for any set of experiments $\I$ on $\V$, if
        \begin{equation} \label{eq: lower bound - size}
            |\mathbf{I}|<\smax{\G}-1, \quad \forall \mathbf{I} \in \I, 
        \end{equation}
        then $|[\G]^d_{\I}|>1$ and $|[\G]^{\sigma}_{\I}|>1$.
    \end{theorem}
    \begin{proof}
        Consider the constructed DGs $\G$ and $\G^*$ in Lemma \ref{lemma: fundamental}.
        Note that $\smax{\G} = c$.
        If Equation \eqref{eq: lower bound - size} holds for a set of experiments $\I$, then Equation \eqref{eq: fundamental} does not hold.
        Therefore, Lemma \ref{lemma: fundamental} implies that $\G' \in [\G]^r_{\I}$, which completes the proof.
    \end{proof}
    \begin{corollary}
        In the worst case, DG $\G$ cannot be learned by any algorithm (adaptive or non-adaptive) that performs experiments with size less than $\smax{\G}-1$ for both scenarios described in Section \ref{sec:desc}.
    \end{corollary}
    \begin{theorem} \label{thm: lower bound - number}
        Consider a set of $n$ vertices denoted by $\V$ and a constant $1< \!c \leq n$. 
        There exists a DG $\G$ over $\V$ with $\smax{\G}=c$ such that for any set of experiments $\I$ on $\V$, if $|\I|< \smax{\G}$, then, $|[\G]^d_{\I}|>1$ and $|[\G]^{\sigma}_{\I}|>1$.
    \end{theorem}
    \begin{proof}
        Consider the constructed DG $\G$ in Lemma \ref{lemma: fundamental}.
        Note that $\smax{\G} = c$.
        Since $|\I|< \smax{\G}$, there exists $Y^* \in \V_c$ such that Equation \eqref{eq: fundamental} does not hold.
        Let $X^*$ be an arbitrary variable in $\V_c \setminus \{Y^*\}$, and let $\G'$ denote the DG obtained by removing the edge $(X^*, Y^*)$ from $\G$.
        In this case, Lemma \ref{lemma: fundamental} implies that $\G' \in [\G]^r_{\I}$, which concludes the proof.
    \end{proof}

    \begin{corollary}
        At least $\smax{\G}$ experiments are required to learn $\G$ in the worst case.
    \end{corollary}

\section{Unbounded-size Experiment Design} \label{sec: unbounded size}
    In this section, we propose a two-stage experiment design algorithm for learning a DG $\G$ (potentially cyclic) when there is no constraint on the size of the designed experiments.
    In the first stage, we design a set of experiments for learning the descendant sets of the variables and the strongly connected components (SCC) of $\G$.
    In the second stage and based on the findings of the first stage, we design further experiments to exactly recover $\G$.
    
    \subsection{Stage 1: Colored Separating System} \label{sec: stage 1}
        In this section, we introduce the first stage of our approach for learning the descendant sets $\{\De{X}{\G}\}_{X\in \V}$ and the set of SCCs $\mathcal{S} = \{\mathbf{S}_1,\dots,\mathbf{S}_k \}$ of $\G$.
        This stage is based on performing experiments on certain subsets of $\V$ that form a \emph{colored separating system}.
        
        \begin{definition}[Colored separating system]\label{sep} \label{def: colored sep sys}
            Suppose $\V = \{X_1,\cdots, X_n\}$ and let $\mathcal{C} = \{C_1,\cdots, C_n\}$ be an arbitrary coloring for $\V$.
            A colored separating system $\I$ on $(\V,\mathcal{C})$ is a collection of subsets of $\V$ such that for every distinct ordered pair of variables $(X_i,X_j)$ in $\V$, if $C_i \neq C_j$, then there exists $\mathbf{I} \in \I$ such that $X_i\in \mathbf{I}$ and $X_j \notin \mathbf{I}$.
        \end{definition} 
        We note that similar definitions have been proposed in the literature.
        For instance, \cite{katona1966separating} introduced \emph{separating systems}, a special case of colored separating system, where $\mathcal{C}$ must contain $n$ different colors.
        
        In the following, we provide an achievable bound on the cardinality of a colored separating system.
        
        \begin{proposition} \label{prp: colored separating system}
            There exists a colored separating system on $(\V,\mathcal{C})$ with at most $2\lceil \log_2(\chi)\rceil$ elements, where $\chi$ is the number of colors in $\mathcal{C}$.
        \end{proposition}
        \begin{proof}
            Suppose $\V = \{X_1,\cdots, X_n\}$ and let $l = \lceil \log_2(\chi)\rceil$. 
            Suppose  $\mathcal{C} = \{C_1,\cdots, C_n\}$, where $C_i \in \{1, \cdots, \chi \}$.
            For $1\leq i \leq l$, let $\N_i$ be the subset of numbers in $\{1,2,\cdots,\chi \}$ whose $i$-th bit in binary representation equals to 1.
            We now construct subsets $\mathbf{I}^1_i,\mathbf{I}^2_i \subseteq \V$ for each $1\leq i \leq l$ as follows:
            
            \begin{equation} \label{eq: colored sep sys}
                \mathbf{I}^1_i = \{X_j \in \V \mid C_j \in \N_i\}, \hspace{1cm}
                \mathbf{I}^2_i = \{X_j \in \V \mid C_j \notin \N_i\}. 
            \end{equation}
            Let $(X_a,X_b)$ be an ordered pair of distinct variables in $\V$ such that $C_a \neq C_b$.
            In this case, there exists $1\leq i \leq l$ such that the $i$-th  bit of $C_a$ and $C_b$ are different in binary representation.
            There are two cases:
            \begin{itemize}
                \item The $i$-bit of $C_a$ in binary representation is 1: In this case $X_a \in \mathbf{I}^1_i$ and $X_b \notin \mathbf{I}^1_i$. 
                \item The $i$-bit of $C_a$ in binary representation is 0: In this case $X_a \in \mathbf{I}^2_i$ and $X_b \notin \mathbf{I}^2_i$. 
            \end{itemize}
            This shows that $\I = \{\mathbf{I}^1_i\}_{i=1}^l \cup \{\mathbf{I}^2_i\}_{i=1}^l$ is a colored separating set on $(\V, \mathcal{C})$. 
            Note that $|\I| = 2l =2\lceil \log_2(\chi)\rceil$. 
        \end{proof}
        \begin{remark}
            The proof of Proposition \ref{prp: colored separating system} is constructive. 
            That is, with Equation \eqref{eq: colored sep sys} we can obtain a colored separating system on $(\V,\mathcal{C})$ with at most $2\lceil \log_2(\chi)\rceil$ elements.
        \end{remark}

        \begin{algorithm}[t]
            \caption{Learning descendant sets and strongly connected components}
            \label{algo: SCC}
            \begin{algorithmic}[1]
                \STATE Learn $\G_r^{obs}$ using observational data
                \STATE $\mathcal{C} \gets $ A vertex coloring for $\G_r^{obs}$
                \STATE $\I \gets $ Construct a colored separating system on $(\V,\mathcal{C})$
                \FOR{$X \in \V$}
                    \STATE $\I_X \gets \{\mathbf{I} \in \I\!:\: X \in \mathbf{I} \}$
                    \STATE Initialize $\D_X$ with an empty set
                    \FOR{$\mathbf{I} \in \I_X$}
                        \STATE Add the elements of $\{Y \in \Ne{X}{\G_r^{obs}} \!:\:\nmarginsep{X}{Y}{P_{do(\mathbf{I})}}{}\}$ to $\D_X$
                    \ENDFOR
                \ENDFOR
                \STATE Construct DG $\mathcal{H}$ by adding directed edges from $X$ to $\D_X$ for each $X \in \V$
                \STATE $\{\De{X}{\G}\}_{X\in \V}, \mathcal{S} = \{\mathbf{S}_1,\dots,\mathbf{S}_k\} \gets $ Compute descendant sets and SCCs of $\mathcal{H}$
                \STATE \textbf{Return} $\{\De{X}{\G}\}_{X\in \V}, \mathcal{S} = \{\mathbf{S}_1,\dots,\mathbf{S}_k\} $
            \end{algorithmic}
        \end{algorithm}
        
        Equipped with Proposition \ref{prp: colored separating system}, we present Algorithm \ref{algo: SCC} for finding the descendant sets and the set of SCCs in $\G$.
        At first, the algorithm learns $\G_r^{obs}$ from observational data using existing methods such as the one proposed by \cite{ghassami2020characterizing}.
        For more information on this step, please see Section \ref{subsec: implementation}.
        In lines 2 and 3, it learns a coloring of $\G_r^{obs}$ and subsequently, it constructs a colored separating system on $(\V,\mathcal{C})$ (using Proposition \ref{prp: colored separating system}).
        One way to color  $\G^{obs}_r$ is by using the \emph{trail-path} algorithm described in \cite{bandyopadhyay2020graph}.

       \begin{example}[Colored separating system] \label{example: colored sep sys}
            Consider DG $\G$ in Figure \ref{fig: 4a} over the set of variables $\V = \{X_1,X_2,X_3,X_4\} \cup \{Y_1,Y_2,Y_3,Y_4\} \cup \{Z_1,Z_2,Z_3,Z_4\}$.
            DG $\G$ is cyclic with three SCCs $\mathbf{S}_1 = \{X_1,X_2,X_3,X_4\}$, $\mathbf{S}_2 = \{Y_1,Y_2,Y_3,Y_4\}$, and $\mathbf{S}_3 = \{Z_1,Z_2,Z_3,Z_4\}$.
            In Scenario 1, i.e., when CIs are equivalent to $d$-separations, Algorithm \ref{algo: SCC} learns the undirected graph $\G_d^{obs}$ from observational data, which is depicted in Figure \ref{fig: 4b}.
            Recall that $\G_d^{obs}$ includes the virtual edges (red edges) and the edges of the skeleton of $\G$ (black edges).
            A coloring for $\G_d^{obs}$ with four colors is shown in Figure \ref{fig: 4b}.
            Specifically, $\{X_2, X_4, Z_2, Z_4\}$, $\{X_1,X_3,Z_1,Z_3\}$, $\{Y_1,Y_3\}$, and $\{Y_2,Y_4\}$ comprise the set of variables with the same color.
            Using this coloring, Proposition \ref{prp: colored separating system} constructs the following colored separating system of size $2\lceil \log_2(4)\rceil = 4$:
            \begin{equation*}
            \begin{split}
                \I = \{
                &\{X_1,X_3,Y_1,Y_3,Z_1,Z_3\},
                \{X_2,X_4,Y_2,Y_4,Z_2,Z_4\},\\
                &\{X_1,X_3,Y_2,Y_4,Z_1,Z_3\},
                \{X_2,X_4,Y_1,Y_3,Z_2,Z_4\}
                \}.
            \end{split}
            \end{equation*}
        \end{example}
        
        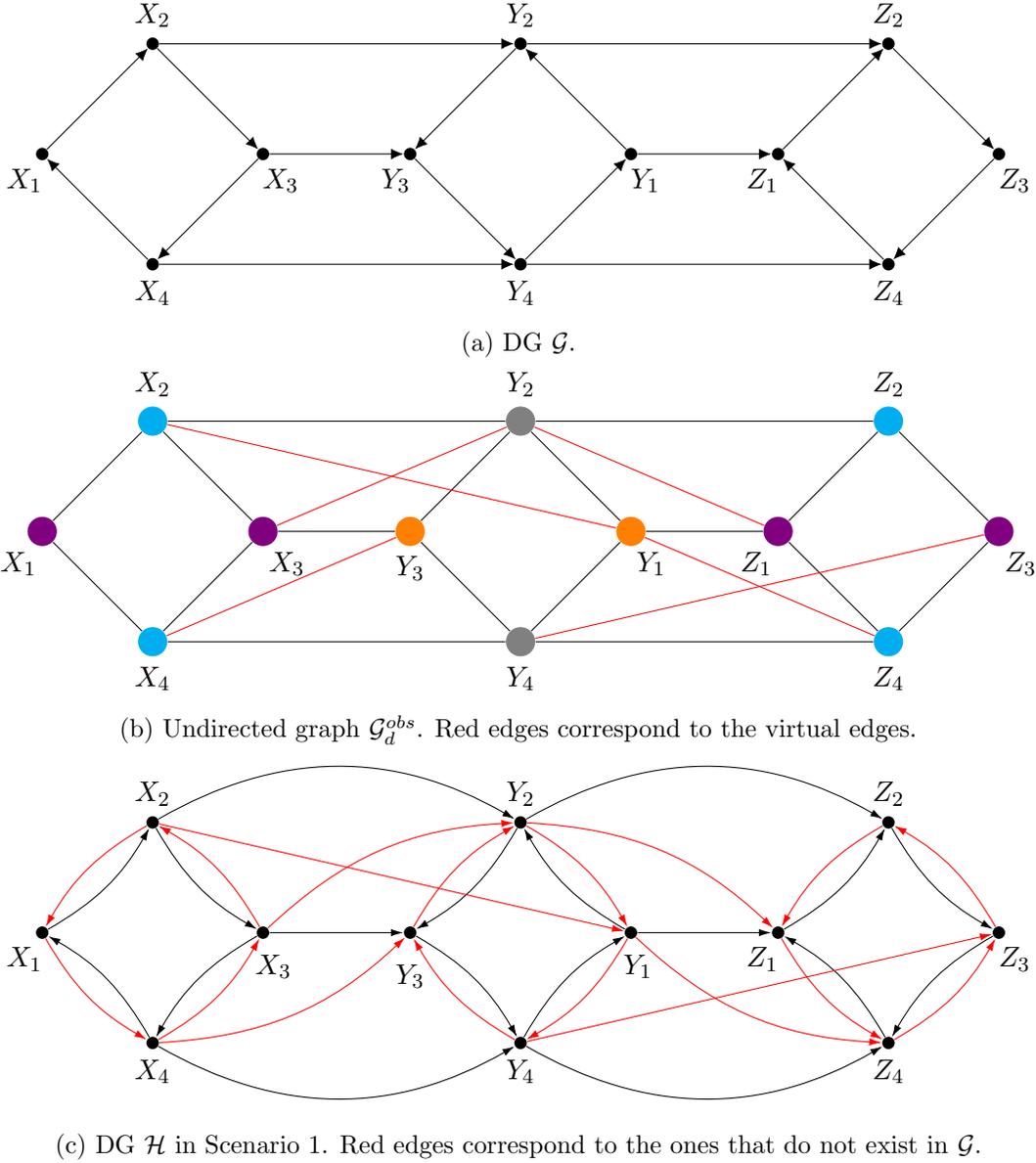
\begin{figure}[t] 
            \centering
            \tikzstyle{block} = [circle, inner sep=1.7pt, fill=black]
            \tikzstyle{c1} = [circle, inner sep=4pt, fill=violet]
            \tikzstyle{c2} = [circle, inner sep=4pt, fill=cyan]
            \tikzstyle{c3} = [circle, inner sep=4pt, fill=orange]
            \tikzstyle{c4} = [circle, inner sep=4pt, fill=gray]
        
            \begin{subfigure}[b]{\textwidth}
                \centering
                \begin{tikzpicture}
                    \tikzset{edge/.style = {->,> = latex',-{Latex[width=1.5mm]}}}
                    % vertices
                    \node[block] (X1) at  (0,0) {};
                    \node[] ()[below left=0 and -.2 of X1]{$X_1$};
                    \node[block] (X2) at  (1.5,1.5) {};
                    \node[] ()[above=0 of X2]{$X_2$};
                    \node[block] (X3) at  (3,0) {};
                    \node[] ()[below right=0 and -.2  of X3]{$X_3$};
                    \node[block] (X4) at  (1.5,-1.5) {};
                    \node[] ()[below=0 of X4]{$X_4$};
        
                    \node[block] (Y1) at  (8,0) {};
                    \node[] ()[below right=0 and -.2 of Y1]{$Y_1$};
                    \node[block] (Y2) at  (6.5,1.5) {};
                    \node[] ()[above=0 of Y2]{$Y_2$};
                    \node[block] (Y3) at  (5,0) {};
                    \node[] ()[below left=0 and -.2 of Y3]{$Y_3$};
                    \node[block] (Y4) at  (6.5,-1.5) {};
                    \node[] ()[below=0 of Y4]{$Y_4$};
        
                    \node[block] (Z1) at  (10,0) {};
                    \node[] ()[below left=0 and -.2 of Z1]{$Z_1$};
                    \node[block] (Z2) at  (11.5,1.5) {};
                    \node[] ()[above=0 of Z2]{$Z_2$};
                    \node[block] (Z3) at  (13,0) {};
                    \node[] ()[below right=0 and -.2  of Z3]{$Z_3$};
                    \node[block] (Z4) at  (11.5,-1.5) {};
                    \node[] ()[below=0 of Z4]{$Z_4$};
                    %edges
                    \draw[edge] (X1) to (X2);
                    \draw[edge] (X2) to (X3);
                    \draw[edge] (X3) to (X4);
                    \draw[edge] (X4) to (X1);
        
                    \draw[edge] (Y1) to (Y2);
                    \draw[edge] (Y2) to (Y3);
                    \draw[edge] (Y3) to (Y4);
                    \draw[edge] (Y4) to (Y1);
        
                    \draw[edge] (Z1) to (Z2);
                    \draw[edge] (Z2) to (Z3);
                    \draw[edge] (Z3) to (Z4);
                    \draw[edge] (Z4) to (Z1);
        
                    \draw[edge] (X2) to (Y2);
                    \draw[edge] (X3) to (Y3);
                    \draw[edge] (X4) to (Y4);

                    \draw[edge] (Y1) to (Z1);
                    \draw[edge] (Y2) to (Z2);
                    \draw[edge] (Y4) to (Z4);
                \end{tikzpicture}
                \caption{DG $\G$.}
                \label{fig: 4a}
            \end{subfigure}
        
            \begin{subfigure}[b]{\textwidth}
                \centering
                \begin{tikzpicture}
                    % vertices
                    \node[c1] (X1) at  (0,0) {};
                    \node[] ()[below left=0 and -.2 of X1]{$X_1$};
                    \node[c2] (X2) at  (1.5,1.5) {};
                    \node[] ()[above=0 of X2]{$X_2$};
                    \node[c1] (X3) at  (3,0) {};
                    \node[] ()[below right=0 and -.2  of X3]{$X_3$};
                    \node[c2] (X4) at  (1.5,-1.5) {};
                    \node[] ()[below=0 of X4]{$X_4$};
        
                    \node[c3] (Y1) at  (8,0) {};
                    \node[] ()[below right=0 and -.2 of Y1]{$Y_1$};
                    \node[c4] (Y2) at  (6.5,1.5) {};
                    \node[] ()[above=0 of Y2]{$Y_2$};
                    \node[c3] (Y3) at  (5,0) {};
                    \node[] ()[below =0 of Y3]{$Y_3$};
                    \node[c4] (Y4) at  (6.5,-1.5) {};
                    \node[] ()[below=0 of Y4]{$Y_4$};
        
                    \node[c1] (Z1) at  (10,0) {};
                    \node[] ()[below left=0 and -.2 of Z1]{$Z_1$};
                    \node[c2] (Z2) at  (11.5,1.5) {};
                    \node[] ()[above=0 of Z2]{$Z_2$};
                    \node[c1] (Z3) at  (13,0) {};
                    \node[] ()[below right=0 and -.2  of Z3]{$Z_3$};
                    \node[c2] (Z4) at  (11.5,-1.5) {};
                    \node[] ()[below=0 of Z4]{$Z_4$};
                    %edges
                    \draw[-] (X1) to (X2);
                    \draw[-] (X2) to (X3);
                    \draw[-] (X3) to (X4);
                    \draw[-] (X4) to (X1);
        
                    \draw[-] (Y1) to (Y2);
                    \draw[-] (Y2) to (Y3);
                    \draw[-] (Y3) to (Y4);
                    \draw[-] (Y4) to (Y1);
        
                    \draw[-] (Z1) to (Z2);
                    \draw[-] (Z2) to (Z3);
                    \draw[-] (Z3) to (Z4);
                    \draw[-] (Z4) to (Z1);
        
                    \draw[-] (X2) to (Y2);
                    \draw[-] (X3) to (Y3);
                    \draw[-] (X4) to (Y4);

                    \draw[-] (Y1) to (Z1);
                    \draw[-] (Y2) to (Z2);
                    \draw[-] (Y4) to (Z4);
        
                    \draw[-, red] (X3) to (Y2);
                    \draw[-, red] (X2) to (Y1);
                    \draw[-, red] (X4) to (Y3);
        
                    \draw[-, red] (Y1) to (Z4);
                    \draw[-, red] (Y2) to (Z1);
                    \draw[-, red] (Y4) to (Z3);
                \end{tikzpicture}
                \caption{Undirected graph $\G_d^{obs}$. Red edges correspond to the virtual edges.}
                \label{fig: 4b}
            \end{subfigure}
        
            \begin{subfigure}[b]{\textwidth}
                \centering
                \begin{tikzpicture}
                    \tikzset{edge/.style = {->,> = latex',-{Latex[width=1mm]}}}
                    % vertices
                    \node[block] (X1) at  (0,0) {};
                    \node[] ()[below left=0 and -.2 of X1]{$X_1$};
                    \node[block] (X2) at  (1.5,1.5) {};
                    \node[] ()[above=0 of X2]{$X_2$};
                    \node[block] (X3) at  (3,0) {};
                    \node[] ()[below right=0.1 and -.3  of X3]{$X_3$};
                    \node[block] (X4) at  (1.5,-1.5) {};
                    \node[] ()[below=0 of X4]{$X_4$};
        
                    \node[block] (Y1) at  (8,0) {};
                    \node[] ()[below right=0.1 and -.3 of Y1]{$Y_1$};
                    \node[block] (Y2) at  (6.5,1.5) {};
                    \node[] ()[above=0 of Y2]{$Y_2$};
                    \node[block] (Y3) at  (5,0) {};
                    \node[] ()[below =0.2 of Y3]{$Y_3$};
                    \node[block] (Y4) at  (6.5,-1.5) {};
                    \node[] ()[below=0 of Y4]{$Y_4$};
        
                    \node[block] (Z1) at  (10,0) {};
                    \node[] ()[below left=0 and -.2 of Z1]{$Z_1$};
                    \node[block] (Z2) at  (11.5,1.5) {};
                    \node[] ()[above=0 of Z2]{$Z_2$};
                    \node[block] (Z3) at  (13,0) {};
                    \node[] ()[below right=0 and -.2  of Z3]{$Z_3$};
                    \node[block] (Z4) at  (11.5,-1.5) {};
                    \node[] ()[below=0 of Z4]{$Z_4$};
                    
                    %edges             
                    \path[->, red] (X2) edge[style = {->,> = latex',-{Latex[width=1mm]}}, bend right=15](X1);
                    \path[->, red] (X3) edge[style = {->,> = latex',-{Latex[width=1mm]}}, bend right=15](X2);
                    \path[->, red] (X4) edge[style = {->,> = latex',-{Latex[width=1mm]}}, bend right=15](X3);
                    \path[->, red] (X1) edge[style = {->,> = latex',-{Latex[width=1mm]}}, bend right=15](X4);
                    \path[->] (X1) edge[style = {->,> = latex',-{Latex[width=1mm]}}, bend right=15](X2);
                    \path[->] (X2) edge[style = {->,> = latex',-{Latex[width=1mm]}}, bend right=15](X3);
                    \path[->] (X3) edge[style = {->,> = latex',-{Latex[width=1mm]}}, bend right=15](X4);
                    \path[->] (X4) edge[style = {->,> = latex',-{Latex[width=1mm]}}, bend right=15](X1);
        
                    \path[->, red] (Y2) edge[style = {->,> = latex',-{Latex[width=1mm]}}, bend left=15](Y1);
                    \path[->, red] (Y3) edge[style = {->,> = latex',-{Latex[width=1mm]}}, bend left=15](Y2);
                    \path[->, red] (Y4) edge[style = {->,> = latex',-{Latex[width=1mm]}}, bend left=15](Y3);
                    \path[->, red] (Y1) edge[style = {->,> = latex',-{Latex[width=1mm]}}, bend left=15](Y4);
                    \path[->] (Y1) edge[style = {->,> = latex',-{Latex[width=1mm]}}, bend left=15](Y2);
                    \path[->] (Y2) edge[style = {->,> = latex',-{Latex[width=1mm]}}, bend left=15](Y3);
                    \path[->] (Y3) edge[style = {->,> = latex',-{Latex[width=1mm]}}, bend left=15](Y4);
                    \path[->] (Y4) edge[style = {->,> = latex',-{Latex[width=1mm]}}, bend left=15](Y1);
                    
                    \path[->, red] (Z2) edge[style = {->,> = latex',-{Latex[width=1mm]}}, bend right=15](Z1);
                    \path[->, red] (Z3) edge[style = {->,> = latex',-{Latex[width=1mm]}}, bend right=15](Z2);
                    \path[->, red] (Z4) edge[style = {->,> = latex',-{Latex[width=1mm]}}, bend right=15](Z3);
                    \path[->, red] (Z1) edge[style = {->,> = latex',-{Latex[width=1mm]}}, bend right=15](Z4);
                    \path[->] (Z1) edge[style = {->,> = latex',-{Latex[width=1mm]}}, bend right=15](Z2);
                    \path[->] (Z2) edge[style = {->,> = latex',-{Latex[width=1mm]}}, bend right=15](Z3);
                    \path[->] (Z3) edge[style = {->,> = latex',-{Latex[width=1mm]}}, bend right=15](Z4);
                    \path[->] (Z4) edge[style = {->,> = latex',-{Latex[width=1mm]}}, bend right=15](Z1);
        
                    \path[->] (X2) edge[style = {->,> = latex',-{Latex[width=1mm]}}, bend left=30](Y2);
                    \path[->, red] (X3) edge[style = {->,> = latex',-{Latex[width=1mm]}}, bend left=20](Y2);
                    \draw[edge] (X3) to (Y3);
                    \path[->] (X4) edge[style = {->,> = latex',-{Latex[width=1mm]}}, bend right=30](Y4);
                    \path[->, red] (X4) edge[style = {->,> = latex',-{Latex[width=1mm]}}, bend right=20](Y3);
                    \path[->, red] (X2) edge[style = {->,> = latex',-{Latex[width=1mm]}}, bend right=0](Y1);

                    \path[->] (Y2) edge[style = {->,> = latex',-{Latex[width=1mm]}}, bend left=30](Z2);
                    \path[->, red] (Y2) edge[style = {->,> = latex',-{Latex[width=1mm]}}, bend left=20](Z1);
                    \draw[edge] (Y1) to (Z1);
                    \path[->] (Y4) edge[style = {->,> = latex',-{Latex[width=1mm]}}, bend right=30](Z4);
                    \path[->, red] (Y1) edge[style = {->,> = latex',-{Latex[width=1mm]}}, bend right=20](Z4);
                    \path[->, red] (Y4) edge[style = {->,> = latex',-{Latex[width=1mm]}}, bend right=0](Z3);
                \end{tikzpicture}
                \caption{DG $\mathcal{H}$ in Scenario 1. Red edges correspond to the ones that do not exist in $\G$.}
                \label{fig: 4c}
            \end{subfigure}
        
            \caption{A running example for our proposed approach (Examples \ref{example: colored sep sys}, \ref{example: H}, and \ref{example: lift sep sys}).}
            \label{fig: 4}
        \end{figure}

        After constructing a colored separating system, Algorithm \ref{algo: SCC} constructs a set $\D_X$ for each $X \in \V$ in lines 4-8 as follows.
        In line 5, $\I_X = \{\mathbf{I}\in \I\!:\: X \in \mathbf{I}\}$ is defined and in line 6, $\D_X$ is initialized with an empty set.
        Based on the following lemma, for any set $\mathbf{I}\subseteq \V$ and each $X \in \mathbf{I}$,  $\De{X}{\G_{\overline{\mathbf{I}}}}$ is learned by performing an experiment on $\mathbf{I}$.
        
        \begin{lemma} \label{lem: learn des}
            For each $X \in \mathbf{I} \subseteq \V$,  $\De{X}{\G_{\overline{\mathbf{I}}}} = \{Y\in \V \!:\:\nmarginsep{X}{Y}{P_{do(\mathbf{I})}}{}\}$.
        \end{lemma}
        \begin{proof}
            We first show that $\De{X}{\G_{\overline{\mathbf{I}}}} = \{Y \in \V \!:\:\nmarginsep{X}{Y}{\G_{\overline{\mathbf{I}}}}{r}\}$.
            \begin{itemize}[leftmargin=*]
                \item Suppose $Y \in \De{X}{\G_{\overline{\mathbf{I}}}}$.
                In this case, there exists a directed path from $X$ to $Y$ in $\G_{\overline{\mathbf{I}}}$ and therefore, $\nmarginsep{X}{Y}{\G_{\overline{\mathbf{I}}}}{r}$.
                \item Suppose $\nmarginsep{X}{Y}{\G_{\overline{\mathbf{I}}}}{r}$.
                This implies that there exists a path $\mathcal{P}$ between $X$ and $Y$ in $\G_{\overline{\mathbf{I}}}$ that does not contain any colliders.
                Note that $X$ does not have any parents in $\G_{\overline{\mathbf{I}}}$ because $X \in \mathbf{I}$.
                Thus, $\mathcal{P}$ must be a directed path from $X$ to $Y$, which implies that $Y \in \De{X}{\G_{\overline{\mathbf{I}}}}$.
            \end{itemize}
            Since $do(\mathbf{I})$ is a full-support hard intervention, the CI assertions in $P_{do(\mathbf{I})}$ are equivalent to $d$-separations or $\sigma$-separations of $\G_{\overline{\mathbf{I}}}$ for Scenario 1 or Scenario 2, respectively.
            Hence, set $\{Y \in \V \!: \nmarginsep{X}{Y}{P_{do(\mathbf{I})}}{}\}$ is equal to $\{Y \in \V \!:\:\nmarginsep{X}{Y}{\G_{\overline{\mathbf{I}}}}{d}\}$ in Scenario 1, and is equal to $\{Y \in \V \!: \nmarginsep{X}{Y}{\G_{\overline{\mathbf{I}}}}{\sigma}\}$ in Scenario 2.
            Therefore, under both  Scenarios 1 or 2,  $\De{X}{\G_{\overline{\mathbf{I}}}} = \{Y \in \V \!:\: \nmarginsep{X}{Y}{P_{do(\mathbf{I})}}{}\}$.
        \end{proof}
        Applying Lemma \ref{lem: learn des}, Algorithm \ref{algo: SCC} adds $\De{X}{\G_{\overline{\mathbf{I}}}} \cap \Ne{X}{\G_r^{obs}}$ to $\D_X$ for each $\mathbf{I} \in \I_X$ in line 8.
        Therefore, at the end of the for loop (lines 7-8), we have
        \begin{equation} \label{eq: Dx}
            \D_X = \left(\bigcup_{\mathbf{I} \in \I_X} \De{X}{\G_{\overline{\mathbf{I}}}}\right) \cap \Ne{X}{\G_r^{obs}}.
        \end{equation}
        Next, we show that $\D_X$ contains $\Ch{X}{\G}$, and it is also a subset of $\De{X}{\G}$.
        
        \begin{lemma} \label{lem: Dx}
            For each $X \in \V$, $\Ch{X}{\G} \subseteq \D_X \subseteq \De{X}{\G}$, where $\D_X$ is defined in \eqref{eq: Dx}.
        \end{lemma}
        \begin{proof}
            For any subset $\mathbf{I} \subseteq \V$, $\De{X}{\G_{\overline{\mathbf{I}}}}$ is a subset of $\De{X}{\G}$. 
            Hence, $\D_X \subseteq \De{X}{\G}$. 
            
            Suppose $Y \in \Ch{X}{\G}$.
            We need to show that $Y \in \D_X$.
            Note that $X$ and $Y$ are neighbors in $\G_r^{obs}$ and, therefore, have different colors in $\mathcal{C}$. 
            Since $\I$ is a colored separating system on $(\V,\mathcal{C})$, there exists $\mathbf{I} \in \I$ such that $X \in \mathbf{I}$ and $Y \notin \mathbf{I}$.
            In this case, $\mathbf{I} \in \I_X$ since $X \in \mathbf{I}$. 
            Furthermore, $Y \in  \De{X}{\G_{\overline{\mathbf{I}}}}$ because $Y$ is a child of $X$ in $\G_{\overline{\mathbf{I}}}$. 
            This implies that $Y \in \D_X$ and therefore, $\Ch{X}{\G} \subseteq \D_X$.
        \end{proof}
        \begin{remark}
            For the algorithm to successfully learn the SCCs, it is crucial that $\D_X$ contains all the children of $X$. 
            As proven in Lemma \ref{lem: Dx}, this is the case because $\I$ is a colored separating system on $(\V, \mathcal{C})$.
            Note that if $\I$ were not a colored separating system on $(\V, \mathcal{C})$, $\D_X$ would still be a subset of $\De{X}{\G}$, but it would not have necessarily contained all the variables in $\Ch{X}{\G}$.
        \end{remark}
        After learning $\D_X$ for all $X\in \V$, a DG $\mathcal{H}$ is constructed over $\V$ by adding directed edges from $X$ to the variables in $\D_X$ for each $X\in \V$ (line 9).
        
        \begin{example}[DG $\mathcal{H}$] \label{example: H}
            Following Example \ref{example: colored sep sys}, consider the graphs in Figure \ref{fig: 4}.
            DG $\mathcal{H}$, which is constructed by adding directed edges from $X$ to $\D_X$ for each $X \in \V$, is depicted in Figure \ref{fig: 4c}.
            For instance, $\D_{X_2} = \{X_1, X_3, Y_1, Y_2\}$.
            In this figure, black edges are the edges that appear in DG $\G$, while red edges do not exist in $\G$. 
        \end{example}
        Observe that DG $\mathcal{H}$ is a super graph of $\G$, where the extra edges in $\mathcal{H}$ appear only from the variables to some of their descendants in $\G$.
        In fact, the following corollary of Lemma \ref{lem: Dx} holds.
        
        \begin{corollary} \label{cor: learn des}
            In Algorithm \ref{algo: SCC}, DG $\G$ and DG $\mathcal{H}$ (the constructed DG in line 9) have the same descendant sets, i.e., for each $X \in \V$, $\De{X}{\mathcal{H}} = \De{X}{\G}$. 
            Accordingly, $\G$ and $\mathcal{H}$ have the same SCCs.
        \end{corollary}
        Note that the second part of Corollary \ref{cor: learn des} is due to the fact that by definition, two variables $X$ and $Y$ are in the same SCC in $\G$ if and only if $X\in \De{Y}{\G}$ and $Y \in \De{X}{\G}$.
        
        Given a DG with $n$ vertices, there exist efficient depth-first search (DFS)-based algorithms, such as \emph{Kosaraju}, for obtaining the descendant sets and the SCCs with the computational complexity of $\mathcal{O}(n^2)$ \citep{sharir1981strong}.
        Applying any of these algorithms to $\mathcal{H}$, Algorithm \ref{algo: SCC} can obtain $\{\De{X}{\G}\}_{X\in \V}$ and the set of SCCs $\mathcal{S} = \{\mathbf{S}_1,\dots,\mathbf{S}_k\}$ of $\G$ in line 10.
        
        \begin{remark} \label{remark: Gobs}
            For the soundness of Algorithm \ref{algo: SCC}, it suffices for $\G_r^{obs}$ to be a super graph of the skeleton of $\G$.
            For instance, if we do not have access to the observational data, Algorithm \ref{algo: SCC} can set $\G_r^{obs}$ to be a complete graph in line 1.
        \end{remark}

        \begin{remark} \label{remark: stop if DAG}
            If $\G$ is a DAG, then $\mathcal{H} = \G$, and it suffices for the algorithm to return $\mathcal{H}$ at the end of this stage. To test for this, the algorithm must check whether $\mathcal{H}$ is a DAG after line 9 in Algorithm \ref{algo: SCC}.
        \end{remark}   
    
    \subsection{Stage 2: Lifted Separating System} \label{sec: algo}
        As we discussed in the previous section, the descendant sets and the set of SCCs $\mathcal{S} = \{\mathbf{S}_1,\dots,\mathbf{S}_k\}$ of $\G$ can be learned by performing $2\lceil\log_2(\chi(\G_r^{obs}))\rceil$ experiments.
        Herein, as the second stage of our approach, we design $\smax{\G}\coloneqq \max(|\mathbf{S}_1|, \cdots, |\mathbf{S}_k|)$ new experiments to learn $\G$.
        In this stage, we perform experiments on certain subsets of $\V$ that form a \emph{lifted separating system}, formally defined in the following.
        
        \begin{definition}[Lifted separating system] \label{def: lifted sep sys}
            Suppose $\mathcal{S} = \{\mathbf{S}_1,\dots,\mathbf{S}_k\}$ is the set of SCCs of a DG $\G$ with the set of vertices $\V$.
            A lifted separating system $\I$ on $(\V, \mathcal{S})$ is a collection of subsets of $\V$  such that for each $i \in \{1,\cdots,k\}$ and $X\in \mathbf{S}_i$, there exists $\mathbf{I}\in \I$ such that $\mathbf{S}_i \setminus \{X\} \subseteq \mathbf{I}$ and $X \notin \mathbf{I}$.
        \end{definition}
        We note that, as far as we know, no similar definition exists in the literature.
        In the following, we provide a method for constructing a lifted separating system with at most $\smax{\G}$ elements.
        
        \begin{proposition} \label{prp: lifted sep sys}
            Suppose $\mathcal{S}= \{\mathbf{S}_1,\dots,\mathbf{S}_k\}$ is the set of SCCs of a DG $\G = (\V,\E)$.
            There exists a lifted separating system on $(\V, \mathcal{S})$ with at most $\smax{\G}$ elements.
        \end{proposition}
        \begin{proof}
            For each $1\leq j \leq k$, suppose $\mathbf{S}_j = \{X^j_1, \cdots X^j_{l_j}\}$, where $l_j = |\mathbf{S}_j|$.
            Also, let $l_{\max} = \max(l_1,\cdots,l_k) = \smax{\G}$. 
            For each $1 \leq i \leq l_{\max}$, we construct subset $\mathbf{I}_i \subseteq \V$ as follows.
            For each $1 \leq j \leq k$ such that $i \leq l_j$, we add $\mathbf{S}_j \setminus \{X^j_i\}$ to $\mathbf{I}_i$. 
            That is,
            
            \begin{equation} \label{eq: lifted sep sys}
                \mathbf{I}_i 
                = \bigcup_{\substack{1 \leq j \leq k \\ \text{ s.t. } i \leq l_j}} (\mathbf{S}_j \setminus \{X^j_i\}).
            \end{equation}
            Next, we show that $\I = \{\mathbf{I}_1,\cdots,\mathbf{I}_{l_{\max}}\}$ is a lifted separating system on $(\V, \mathcal{S})$.
            Note that $|\I| = \smax{\G}$.
            Suppose $j\in \{1,\cdots,k\}$ and $X^j_i \in \mathbf{S}_j$, where $ 1\leq i \leq l_j$.
            We need to show that there exists $\mathbf{I}\in \I$ such that $\mathbf{S}_j \setminus \{X^j_i\} \subseteq \mathbf{I}$ and $X^j_i \notin \mathbf{I}$.
           $\mathbf{I} = \mathbf{I}_i$ satisfies this property because $\mathbf{I}_i \cap \mathbf{S}_j = \mathbf{S}_j \setminus \{X^j_i\}$.
            Hence, $\I$ is a lifted separating system on $(\V, \mathcal{S})$ with size $\smax{\G}$.
        \end{proof}
        \begin{remark}
            The proof of Proposition \ref{prp: lifted sep sys} is constructive.
            Given the set of SCCs, Equation \eqref{eq: lifted sep sys} provides a lifted separating system on $(\V, \mathcal{S})$ with at most $\smax{\G}$ elements.
        \end{remark}
        
        \begin{example}[Lifted separating system] \label{example: lift sep sys}
            Consider DG $\G$ in Figure \ref{fig: 4} with three SCCs $\mathbf{S}_1 = \{X_1,X_2,X_3,X_4\}$, $\mathbf{S}_2 = \{Y_1,Y_2,Y_3,Y_4\}$, and $\mathbf{S}_3 = \{Z_1,Z_2,Z_3,Z_4\}$.
            Using Equation \eqref{eq: lifted sep sys} in the proof of Proposition \ref{prp: lifted sep sys}, we can construct the following lifted separating system of size $\smax{\G}= \max(|\mathbf{S}_1|, |\mathbf{S}_2|, |\mathbf{S}_3|) =4$.
            \begin{equation*}
                \begin{split}
                    \I = \{ &
                    \{X_2,X_3,X_4, Y_2,Y_3,Y_4, Z_2,Z_3,Z_4\},
                    \{X_1,X_3,X_4, Y_1,Y_3,Y_4, Z_1,Z_3,Z_4\},\\&
                    \{X_1,X_2,X_4, Y_1,Y_2,Y_4, Z_1,Z_2,Z_4\},
                    \{X_1,X_2,X_3, Y_1,Y_2,Y_3, Z_1,Z_2,Z_3\}
                    \}.
                \end{split}
            \end{equation*}
        \end{example}
        We present Algorithm \ref{algo: main} for learning DG $\G$ that takes the ancestor sets\footnote{Algorithm \ref{algo: SCC} returns the descendant sets which can be used to obtain the ancestor sets.} $\{\Anc{X}{\G}\}_{X\in \V}$ and the set of SCCs $\mathcal{S}= \{\mathbf{S}_1,\dots,\mathbf{S}_k\}$ of $\G$ as inputs.
        The algorithm constructs a lifted separating system $\I$ in line 2 and initializes a DG $\hat{\G}$ on $\V$ with no edges in line 3.

        \begin{algorithm}[t]
            \caption{Learning a DG $\G$}
            \label{algo: main}
            \begin{algorithmic}[1]
                \STATE \textbf{Input: } $\{\Anc{X}{\G}\}_{X\in \V}, \mathcal{S}= \{\mathbf{S}_1,\dots,\mathbf{S}_k\}$
                \STATE $\I \gets $ Construct a lifted separating system on $(\V, \mathcal{S})$
                \STATE Initialization: $\hat{\G} \gets (\V,\hat{\E}=\varnothing)$ 
                \FOR{$i$ from $1$ to $k$}
                    \FOR{$X \in \mathbf{S}_i$}
                        \STATE $\mathbf{I} \gets$ An element of $\I$ that contains $\mathbf{S}_i \setminus \{X\}$ but does not contain $X$
                        \FOR{$Y \in \mathbf{S}_i\setminus \{X\}$}
                            \STATE Add $(Y,X)$ to $\hat{\E}$ if $\nmarginsep{X}{Y}{P_{do(\mathbf{I})}}{}$
                        \ENDFOR
                        \FOR{$Y \in \Anc{X}{\G} \setminus \mathbf{S}_i$}
                            \STATE Add $(Y,X)$ to $\hat{\E}$ if $\nsep{X}{Y}{\Anc{X}{\G} \setminus ( \mathbf{S}_i \cup \{Y\})}{P_{do(\mathbf{I})}}{}$
                        \ENDFOR
                    \ENDFOR
                \ENDFOR
                \STATE \textbf{Return} $\hat{\G}$
            \end{algorithmic}
        \end{algorithm}
        
        Suppose $X$ is an arbitrary variable in $\mathbf{S}_i$, where $1 \leq i \leq k$ (the for loops in lines 4 and 5).
        Since $\I$ is a lifted separating system on $(\V, \mathcal{S})$, there exists $\mathbf{I} \in \I$ that contains $\mathbf{S}_i \setminus \{X\}$ but not $X$ (line 6).
        By performing an experiment on $\mathbf{I}$ and using the following two lemmas, the algorithm finds the parents of $X$ in lines 7-10.
        
        \begin{lemma} \label{lem: Pa(X) part 1}
            Suppose $Y\in \mathbf{S}_i\setminus \{X\}$ and $\mathbf{I}\subseteq \V \setminus \{X\}$ such that $\mathbf{S}_i \setminus \{X\} \subseteq \mathbf{I}$.
            Then, $Y \in \Pa{X}{\G}$ if and only if $\nmarginsep{X}{Y}{P_{do(\mathbf{I})}}{}$. 
        \end{lemma}
        \begin{proof}
            Recall that $\mathbf{S}_i = \SCC{X}{\G}$.
            Since $Y \in \mathbf{I}$, Lemma \ref{lem: learn des} implies that $\nmarginsep{X}{Y}{P_{do(\mathbf{I})}}{}$ if and only if $X \in \De{Y}{\G_{\overline{\mathbf{I}}}}$.
            
            \textit{Sufficient part: }
            If $Y \in \Pa{X}{\G}$, then $Y \in \Pa{X}{\G_{\overline{\mathbf{I}}}}$ since $X \notin \mathbf{I}$.
            Thus, $X \in \De{Y}{\G_{\overline{\mathbf{I}}}}$.
            
            \textit{Necessary part}: 
            If $X \in \De{Y}{\G_{\overline{\mathbf{I}}}}$, then there exists a directed path from $Y$ to $X$ in $\G_{\overline{\mathbf{I}}}$.
            We now show that there exists no directed path from $Y$ to $X$ in $\G_{\overline{\mathbf{I}}}$ with length larger than 1.
            Suppose not and let $\mathcal{P} = (Y, Z_1,\cdots, Z_t, X)$ be a directed path from $Y$ to $X$ in $\G_{\overline{\mathbf{I}}}$, where $t \geq 1$.
            In this case, $\Pa{Z_1}{\G_{\overline{\mathbf{I}}}}$ is non-empty since $Y$ is in it.
            Hence, $Z_1 \notin \mathbf{I}$ and $Z_1 \notin \mathbf{S}_i$ because $\mathbf{S}_i \setminus \{X\} \subseteq \mathbf{I}$.
            This implies that
            \begin{itemize}
                \item $Z_1 \in \Anc{X}{\G}$ because of the directed path $(Z_1,\cdots, Z_t,X)$, and
                \item $Z_1 \in \De{X}{\G}$ because $Y \in \De{X}{\G}$ and $Z_1 \in \Ch{Y}{\G}$.
            \end{itemize}
            This shows that $Z_1 \in \SCC{X}{\G} = \mathbf{S}_i$, which is a contradiction.
            Hence, there exists no directed path from $Y$ to $X$ in $\G_{\overline{\mathbf{I}}}$ with length larger than 1.
            Therefore, $Y \in \Pa{X}{\G}$.
        \end{proof}
        Applying Lemma \ref{lem: Pa(X) part 1}, Algorithm \ref{algo: main} finds the parents of $X$ which belong to $\mathbf{S}_i\setminus \{X\}$ in lines 7-8.
        
        Since $\Pa{X}{\G} \subseteq \Anc{X}{\G}$, the parents of $X$ are either in $\mathbf{S}_i \setminus \{X\}$ or $\Anc{X}{\G} \setminus \mathbf{S}_i$.
        The following lemma shows how the algorithm finds the parents of $X$, which belong to $\Anc{X}{\G} \setminus \mathbf{S}_i$.  
        
        \begin{lemma}\label{lem: Pa(X) part 2}
            Suppose $Y \in \Anc{X}{\G} \setminus \mathbf{S}_i$ and $\mathbf{I}\subseteq \V \setminus \{X\}$ such that $\mathbf{S}_i \setminus \{X\} \subseteq \mathbf{I}$. 
            In this case, $Y \in \Pa{X}{\G}$ if and only if $\nsep{X}{Y}{\Anc{X}{\G} \setminus (\mathbf{S}_i \cup \{Y\})}{P_{do(\mathbf{I})}}{}$. 
        \end{lemma}
        \begin{proof}
            Let $\Z = \Anc{X}{\G} \setminus (\mathbf{S}_i \cup \{Y\})$. 
            We note that $\nsep{X}{Y}{\Z}{P_{do(\mathbf{I})}}{}$ if and only if $\nsep{X}{Y}{\Z}{\G_{\overline{\mathbf{I}}}}{r}$.
            
            \textit{Sufficient part: }
            If $Y \in \Pa{X}{\G}$, then $Y \in \Pa{X}{\G_{\overline{\mathbf{I}}}}$ since $X \notin \mathbf{I}$.
            Thus, $\nsep{X}{Y}{\Z}{\G_{\overline{\mathbf{I}}}}{r}$.
            
            \textit{Necessary part}: 
            Suppose $Y \notin \Pa{X}{\G}$. 
            In this case, $Y \notin \Ch{X}{\G}$ because $Y \in \Anc{X}{\G} \setminus \mathbf{S}_i$.
            We need to show that $\sep{X}{Y}{\Z}{\G_{\overline{\mathbf{I}}}}{r}$.
            Let $\mathcal{P}=(X, Z_1,\cdots,Z_t,Y)$ be a path in $\G_{\overline{\mathbf{I}}}$ between $X$ and $Y$. 
            Note that $t\geq 1$ because $Y \notin \Pa{X}{\G} \cup \Ch{X}{\G}$.  
            We have the following cases:
            \begin{itemize}
                \item $X \gets Z_1$ and $Z_1 \notin \mathbf{S}_i$: 
                Then, $Z_1$ $r$-blocks $\mathcal{P}$ because $Z_1 \in \Pa{X}{\G} \setminus \mathbf{S}_i \subseteq \Z$ and $X \notin \SCC{Z_1}{\G_{\overline{\mathbf{I}}}}$.
                \item $X \gets Z_1$ and $Z_1 \in \mathbf{S}_i$: 
                Then, $Z_1 \in \mathbf{I}$ and  $\Pa{Z_1}{\G_{\overline{\mathbf{I}}}} = \varnothing$. 
                Hence, $t \geq 2$ and $Z_1 \to Z_2$. 
                Note that $Y \in \Anc{Z_1}{\G}$ since $Y \in \Anc{X}{\G}$ and $Z_1 \in \mathbf{S}_i$. 
                Moreover, $Y \notin \mathbf{S}_i = \Anc{Z_1}{\G} \cap \De{Z_1}{\G}$. 
                Therefore, $Y \notin \De{Z_1}{\G}$ and $\mathcal{P}$ contains a collider. 
                Let $Z_j$ be the first collider on $\mathcal{P}$.
                Note that $j \geq 2$ and $Z_j \notin \mathbf{S}_i$ because the variables in $\mathbf{S}_i \setminus \{X\}$ do not have any parents in $\G_{\overline{\mathbf{I}}}$. 
                Furthermore, $Z_j \in \De{Z_1}{\G} = \De{X}{\G}$.
                Hence, $Z_j \notin \Anc{\Z}{\G_{\overline{\mathbf{I}}}}$ and therefore,  $Z_j$ $r$-blocks $\mathcal{P}$.
                \item $X \to Z_1$:
                This case is similar to the previous case. 
                $Y \notin \De{X}{\G}$ because $Y \notin \mathbf{S}_i$. Hence, $\mathcal{P}$ contains a collider.
                Let $Z_j$ be the first collider on $\mathcal{P}$.
                $Z_j \notin \mathbf{S}_i$ because the variables in $\mathbf{S}_i \setminus \{X\}$ do not have any parents in $\G_{\overline{\mathbf{I}}}$. 
                Furthermore, $Z_j \in \De{X}{\G}$.
                Hence, $Z_j \notin \Anc{\Z}{\G_{\overline{\mathbf{I}}}}$ and therefore,  $Z_j$ $r$-blocks $\mathcal{P}$.
            \end{itemize}
            In all of the aforementioned cases, $\mathcal{P}$ is $r$-blocked which shows that $\sep{X}{Y}{\Z}{\G_{\overline{\mathbf{I}}}}{r}$.
        \end{proof}
        Applying Lemma \ref{lem: Pa(X) part 2}, Algorithm \ref{algo: main} finds the rest of the parents of $X$ in lines 9-10.
        Hence, by the time the algorithm terminates, all the parents of $X$ are added to $\hat{\G}$, and $\hat{\G}$ will equal $\G$.
        
        In Section \ref{sec: stage 1}, we showed that the descendant sets and SCCs of a DG $\G$ can be learned by performing experiments on the elements of a colored separating system. 
        Herein, we showed that using the information about the descendant sets and SCCs,  $\G$ can be recovered by performing experiments on the elements of a lifted separating system.
        Moreover, we provided Propositions \ref{prp: colored separating system} and \ref{prp: lifted sep sys} for constructing separating systems and lifted separating systems, respectively, which imply the following.
        \begin{corollary}\label{cor: upper bound}
            Algorithms \ref{algo: SCC} and \ref{algo: main} together can learn a DG $\G$ with $n$ vertices with at most
            \begin{equation} \label{eq: upper bound 1}
                2\lceil\log_2(\chi(\G_r^{obs}))\rceil + \smax{\G}
            \end{equation}
            experiments.
            Comparing this with the lower bound in Theorem \ref{thm: lower bound - number}, the proposed approach is order-optimal in terms of the number of experiments up to an additive logarithmic term.
        \end{corollary}

        \begin{remark} \label{remark: upper bound DAG}
            When $\G$ is a DAG, the first stage of the algorithm uniquely learns $\G$ (Remark \ref{remark: stop if DAG}).
            Hence, the algorithm performs $2\lceil\log_2(\chi(\G_r^{obs}))\rceil$ experiments.
        \end{remark}

\section{Bounded-size Experiment Design} \label{sec: bounded size}
    In the previous sections, we did not impose any constraint on the size of experiments, and our algorithm was allowed to perform experiments with arbitrary sizes.
    In practice, performing \emph{large-sized} experiments may not be possible or too costly.
    In this section, we study the experiment design problem with a constraint on the size of the experiments.
    Formally, our goal is to design a collection of subsets, denoted by $\I$, such that $[\G]^r_{\I} = \{\G\}$ (i.e., $\G$ can be learned by performing experiments on the elements of $\I$), where the size of each $\mathbf{I} \in \I$ is upper bounded by a constant number $M <n$ (i.e., $|\mathbf{I}|\leq M$).
    It is noteworthy that this problem was previously studied for acyclic causal graphs \citep{shanmugam2015learning,lindgren2018experimental}.

    \begin{remark}
        As proved in Theorem \ref{thm: lower bound - size}, it is necessary to perform some experiments with size at least $\smax{\G}-1$ to learn a DG $\G$ in the worst case.
        Hence, the upper bound $M$ cannot be smaller than $\smax{\G}-1$.
    \end{remark}
    We will modify the two stages of our proposed method (introduced in Sections \ref{sec: stage 1} and \ref{sec: algo}) in order to accommodate the new constraint that the size of the experiments is bounded by a constant $M\geq \smax{\G}-1$.

    \subsection{Stage 1: \texorpdfstring{$(n,M)$}{}-separating System}
        In the first stage, instead of learning $\G^{obs}_r$ and constructing a colored separating system, we construct an $(n,M)$-separating system, formally defined by \cite{shanmugam2015learning} as follows.
        
        \begin{definition}[$(n,M)$-separating system]
            An $(n,M)$-separating system $\I$ on $\V$ is a collection of subsets of $\V$ such that $|\mathbf{I}| \leq M$ for each $\mathbf{I} \in \I$, and for every ordered pair of distinct variables $(X,Y)$ in $\V$ there exists $\mathbf{I} \in \I$ such that $X\in \mathbf{I}$ and $Y \notin \mathbf{I}$.
        \end{definition}
        \cite{shanmugam2015learning} also provided an achievable bound on the cardinality of an $(n,M)$-separating system.
    
        \begin{proposition}[\citealt{shanmugam2015learning}]
            There exists an $(n,M)$-separating system on $\V$ with at most $\lceil\frac{n}{M}\rceil \lceil\log_{\lceil\frac{n}{M}\rceil} n\rceil$ elements.
        \end{proposition}
        \begin{remark}
            \cite{shanmugam2015learning} provided a constructive proof for this proposition which allows us to obtain an $(n,M)$-separating system on $\V$ with at most $\lceil\frac{n}{M}\rceil \lceil\log_{\lceil\frac{n}{M}\rceil} n\rceil$ elements.
        \end{remark}
        It suffices to modify lines 1-3 of Algorithm \ref{algo: SCC} by setting $\I$ to be an $(n,M)$-separating system on $\V$ and leaving the rest of the algorithm unchanged.
        It is straightforward to verify that the modified algorithm obtains $\{\De{X}{\G}\}_{X\in \V}$ and the set of SCCs $\mathcal{S} = \{\mathbf{S}_1,\dots,\mathbf{S}_k\}$ of $\G$ by performing experiments on the elements of $\I$.

    \subsection{Stage 2: Bounded Lifted Separating System}
        Algorithm \ref{algo: main} remains unchanged for this stage except that in line 2 of Algorithm \ref{algo: main}, we need to construct a lifted separating system on $(\V,\mathcal{S})$ such that the size of the elements of $\I$ does not exceed $M$.
        
        \begin{theorem} \label{thm: upper bound - bounded size}
            Suppose $\smax{\G}-1 \leq M$.
            There exists a lifted separating system $\I$ on $(\V,\mathcal{S})$ such that for each $\mathbf{I} \in \I$, $|\mathbf{I}| \leq M$, and $|\I| \leq \smax{\G}(1+ \lfloor \frac{n-\smax{\G}-1}{M-\smax{\G}+2}\rfloor)$.
        \end{theorem}
        \begin{proof}
            For each $1\leq j \leq k$, suppose $\mathbf{S}_j = \{X^j_1, \cdots X^j_{l_j}\}$, where $l_j = |\mathbf{S}_j|$.
            Also, let $l_{\max} = \max(l_1,\cdots,l_k) = \smax{\G}$ and $t = \lfloor \frac{n-l_{\max}-1}{M-l_{\max} +2}\rfloor$.
            
            Let us fix an $1 \leq i \leq l_{\max}$. 
            Consider set $\A = \{j \mid 1 \leq j \leq k, i \leq l_j\}$ which is the set of $j$s that variables $X_i^j$s are defined.
            Furthermore, for each $j \in \A$, we define $\B_j = \mathbf{S}_j \setminus \{X_i^j \}$ and $b_j = |\B_j| = l_j-1$.
            Note that $\B_j$s are disjoint and $b_j \leq l_{\max}-1 \leq M$.
            Next, we will introduce $t+1$ subsets (we call them bins) $\mathbf{I}_1,\cdots, \mathbf{I}_{t+1}$ of $\V$, each with size at most $M$, such that for each $j \in \A$, there exists $\mathbf{I} \in \{\mathbf{I}_1,\cdots, \mathbf{I}_{t+1}\}$ such that $\B_j \subseteq \mathbf{I}$ but $X_i^j \notin \mathbf{I}$.
            It is noteworthy that this problem is a special case of \emph{bin-packing problem}.
            For simplicity, suppose $\A = \{j_1, \cdots, j_a\}$, where $a=|\A|$.
            We initialize the bins with empty sets.
            Then, we add $\B_j$s to them in a greedy manner such that the size of bins remains less than $M$.
            That is, we first add the variables in $\B_{j_1}$ to $\mathbf{I}_1$.
            Note that this is feasible since $|\B_{j_1}| \leq M$.
            Then, we add $\B_{j_2}$ to the first feasible bin, i.e., the first bin, such that its size remains less than $M$ after adding the elements of $\B_{j_2}$ to it.
            We subsequently add the elements of $\B_j$s to the first feasible bin.
            It is left to show that there always exists a feasible bin during this process.
            Suppose $\B_{j_1}, \cdots, \B_{j_x}$ are already placed in the bins, where $1 \leq x < a$, and we want to find a feasible bin for $\B_{j_{x+1}}$.
            Assume by contradiction that there is no feasible bin for $\B_{j_{x+1}}$. 
            This shows that adding $\B_{j_{x+1}}$ to any bin results in a bin with at least $M+1$ elements.
            Hence,
            \begin{equation} \label{eq: proof 1}
                (t+1)(M - b_{j_{x+1}} +1) \leq b_{j_1}+\cdots +b_{j_x}.
            \end{equation}
            On the other hand, $\B_{j_1} \cup \cdots \cup \B_{j_x}$ does not intersect with $\B_{j_{x+1}}$ and does not include any of the variables in $\{X_i^{j_1}, \cdots , X_i^{j_{x+1}}\}$. 
            Hence, 
            \begin{equation} \label{eq: proof 2}
                b_{j_1}+\cdots +b_{j_x} \leq n-(b_{j_{x+1}} + x+1).
            \end{equation}
            Note that $b_{j_{x+1}} \leq l_{\max}-1$ and $x\geq 1$.
            Hence, Equations \eqref{eq: proof 1} and \eqref{eq: proof 2} imply that
            \begin{equation*}
                 \lfloor \frac{n-l_{\max}-1}{M-l_{\max} +2}\rfloor +1 = t+1 \leq \frac{n-(b_{j_{x+1}} + x+1)}{M - b_{j_{x+1}} +1} \leq \frac{n-l_{\max}-1}{M - l_{\max} +2}, 
            \end{equation*}
            which is a contradiction. 
            This shows that it is feasible to add all the $\B_j$s to the bins in a greedy manner, and therefore, the constructed $t+1$ subsets satisfy our claim.
            
            Finally, if we repeat the whole process for each $1 \leq i \leq l_{\max}$, the constructed subsets will form a lifted separating system.
            Note that the total number of subsets will equal $l_{\max}(1+t)$, which is our desired bound.
        \end{proof}
        Equipped with Theorem \ref{thm: upper bound - bounded size}, we can obtain a lifted separating system such that the size of its elements is bounded by $M$.
        Moreover, by setting $M= \smax{\G}-1$ in Theorem \ref{thm: upper bound - bounded size}, we get the following notable corollary.
        
        \begin{corollary} \label{cor: Thm 1 is tight}
            DG $\G$ can be learned by performing experiments with size at most $\smax{\G}-1$.
            Hence, the lower bound in Theorem \ref{thm: lower bound - size} is tight.
        \end{corollary}
        To sum up this section, our algorithms can learn a DG $\G$ with $n$ vertices by performing at most
        \begin{equation} \label{eq: upper bound 2}
            \lceil\frac{n}{M}\rceil \lceil\log_{\lceil\frac{n}{M}\rceil} n\rceil+ \smax{\G}(1+ \lfloor \frac{n-\smax{\G}-1}{M-\smax{\G}+2}\rfloor)
        \end{equation}
        experiments with size at most $M$, where $\smax{\G}-1 \leq M < n$.

        \begin{remark}
            Similar to Remark \ref{remark: upper bound DAG}, when $\G$ is a DAG, the first stage of the algorithm uniquely learns $\G$.
            Hence, the algorithm only performs $\lceil\frac{n}{M}\rceil \lceil\log_{\lceil\frac{n}{M}\rceil} n\rceil$ experiments.
        \end{remark}

\section{Simulation Results} \label{sec: simulation}
    In this section, we evaluate the performance of the proposed method over random graphs generated from a variant of stochastic block models (SBMs).\footnote{Our codes are available at \url{https://github.com/Ehsan-Mokhtarian/cyclic_experiment_design}.}
    
    \subsection{Graph Generation}
        In an SBM($n,p,b$), a graph $\G$ with $n$ vertices is generated as follows:
        the variables are randomly partitioned into $\lceil n/b\rceil$ blocks: $\B_1,\cdots,\B_{\lceil n/b\rceil}$, where $|\B_i|=b$ for $1\leq i\leq \lceil n/b\rceil-1$.
        For two variables in the same block, there can exist an edge in both directions, each with probability $p$.
        For two variables in different blocks, there can be an edge between them with probability $p$ only in one direction.
        That is, directed edge $(X,Y)$ exists with probability $p$ when $X \in \B_i$ and $Y \in \B_j$, where $1 \leq i \!<\! j\! \leq \! \lceil n/b\rceil$.
        This means that the variables in each SCC belong to the same block, and $b$ is a surrogate for $\smax{\G}$.

    \subsection{Data Generation}
        For each graph, synthetic data sets from observational and interventional distributions were generated with a finite number of samples and fed to our proposed algorithm.
        The observational samples were generated using a linear SCM where each variable $X$ is a linear combination of its parents plus an exogenous noise variable $\epsilon_X$; the coefficients were chosen uniformly at random from $[-1.5,-1] \cup [1,1.5]$, and $\epsilon_X$ was generated at random according to $\mathcal{N}(0,\sigma_X^2)$, where $\sigma_X$ is selected uniformly at random from $[\sqrt{0.5}, \sqrt{1.5}]$.
        To generate interventional samples for an experiment on a subset $\mathbf{I} \subseteq \V$, the equation of each variable in $\V \setminus \mathbf{I}$ remained unchanged, and the equation of each variable $X \in \mathbf{I}$ was replaced by $X = \epsilon_X$, where $\epsilon_X$ had the same distribution as in the original SCM.
        
    \subsection{Implementation Details} \label{subsec: implementation}
        For the simulations of this section, we used the structure learning algorithm in \cite{mokhtarian2021learning} to learn $\G^{obs}_r$, as it is scalable to large graphs.
        We note that due to Remark \ref{remark: Gobs}, the algorithm does not need to be \emph{complete} (even for DAGs), as we just need $\G^{obs}_r$ to be a supergraph of the skeleton of $\G$.
        Hence, we can exploit any constraint-based causal discovery from observational data method that is \emph{sound} (but not necessarily \emph{complete}) for DAGs to learn $\G^{obs}_r$.
        
        To color $\G^{obs}_r$, we applied \emph{trail-path} algorithm in \cite{bandyopadhyay2020graph}.
        To find the descendant sets and the strongly connected components of $\mathcal{H}$ in line 9 of Algorithm \ref{algo: SCC}, we used the predefined function \textit{conncomp} in MATLAB.
        Finally, we used Fisher Z-transformation with a significance level of $0.01$ to perform conditional independence tests.
        
        % All of the experiments were run in MATLAB on a MacBook Pro laptop equipped with a 1.7 GHz Quad-Core Intel Core i7 processor and 16GB, 2133 MHz, LPDDR3 RAM.

    \subsection{Results}
        In Figure \ref{fig: simulations}, we report the number of experiments performed by our proposed method and the accuracy of the learned graphs when the underlying true graphs are generated randomly from SBM($n,p,b$).
        Each point on the plots is reported as the average of 50 runs with a $90\%$ confidence interval.
        We measured the accuracy of the recovered DGs by normalized structural hamming distance (SHD/$n$) and F1-scores, which we formally define in Subsection \ref{subsec: metrics}.
        
        \begin{figure}[t] 
            \centering
            \captionsetup{justification=centering}
            \begin{subfigure}{.35\textwidth}
                \centering
                \includegraphics[width=\textwidth]{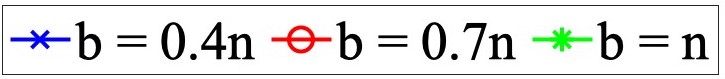}
            \end{subfigure}
            
            \begin{subfigure}{\textwidth}
                \centering
                \includegraphics[width=0.32\textwidth]{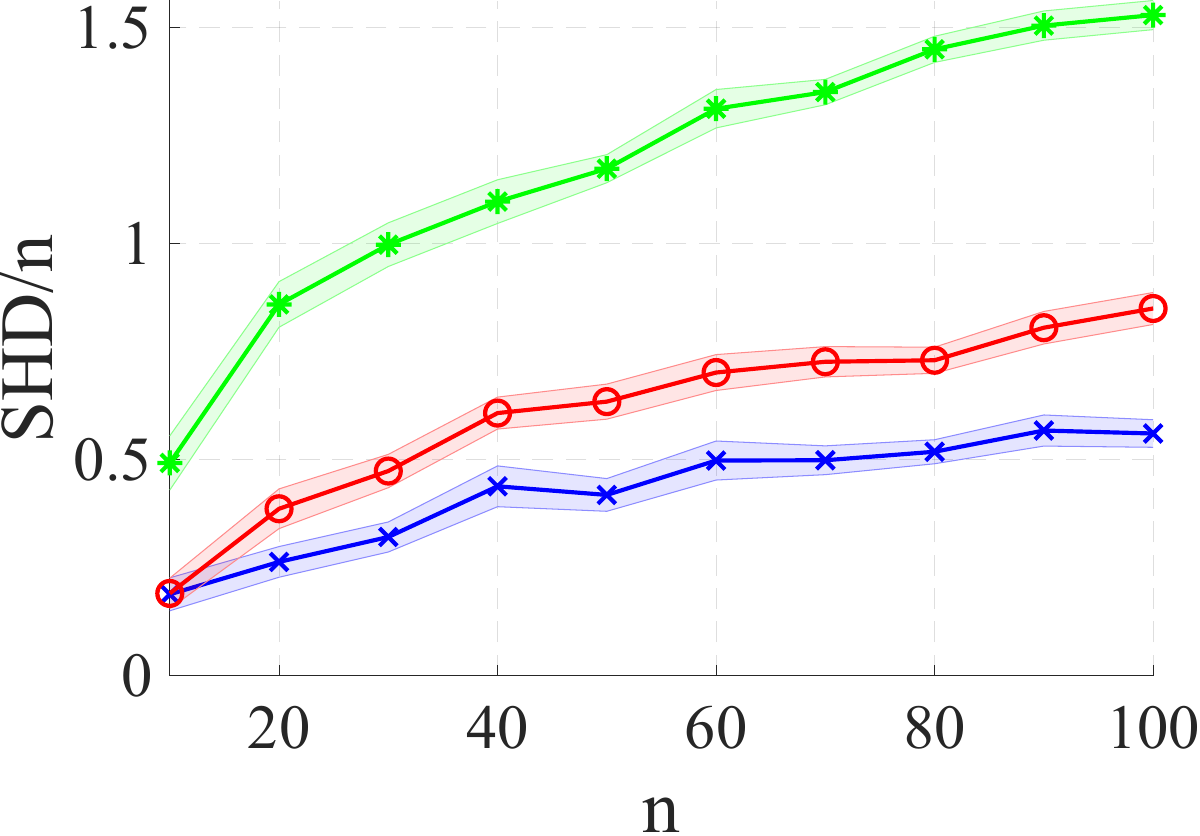}\hfill
                \includegraphics[width=0.32\textwidth]{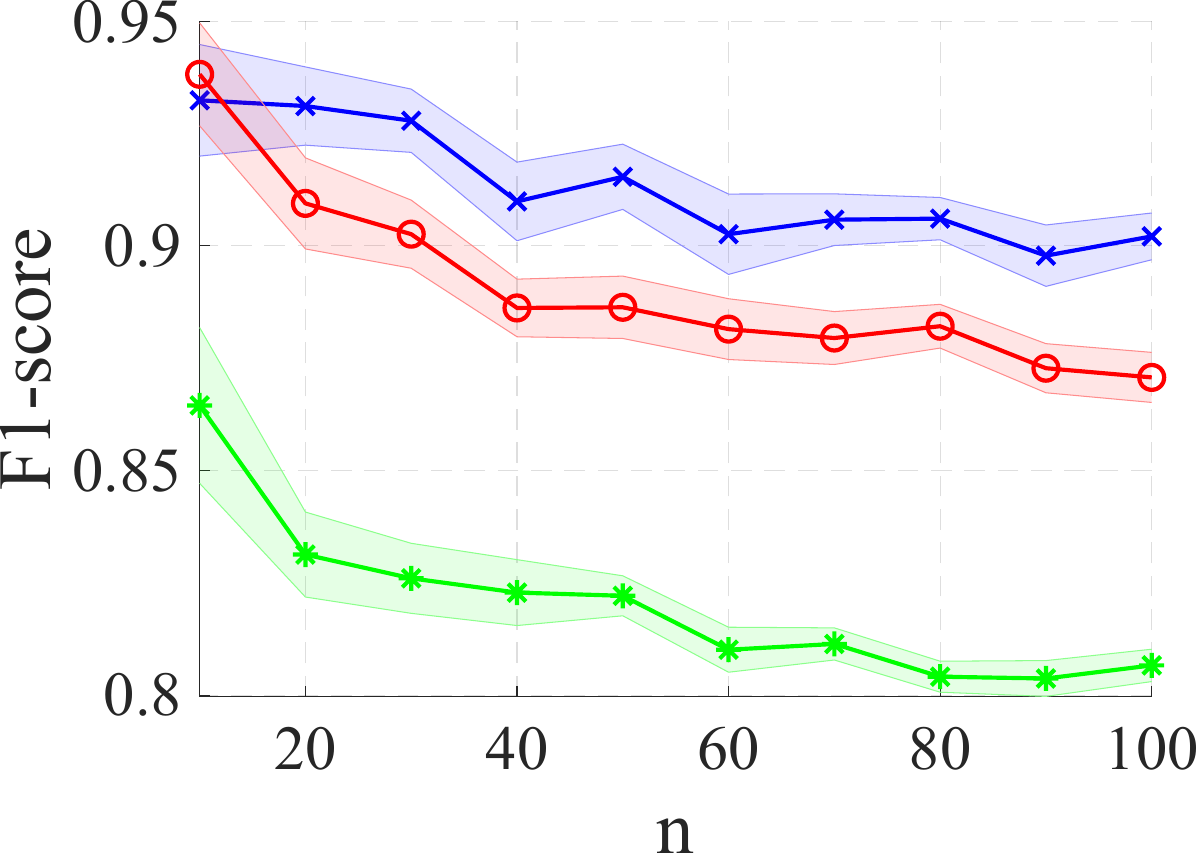}\hfill
                \includegraphics[width=0.32\textwidth]{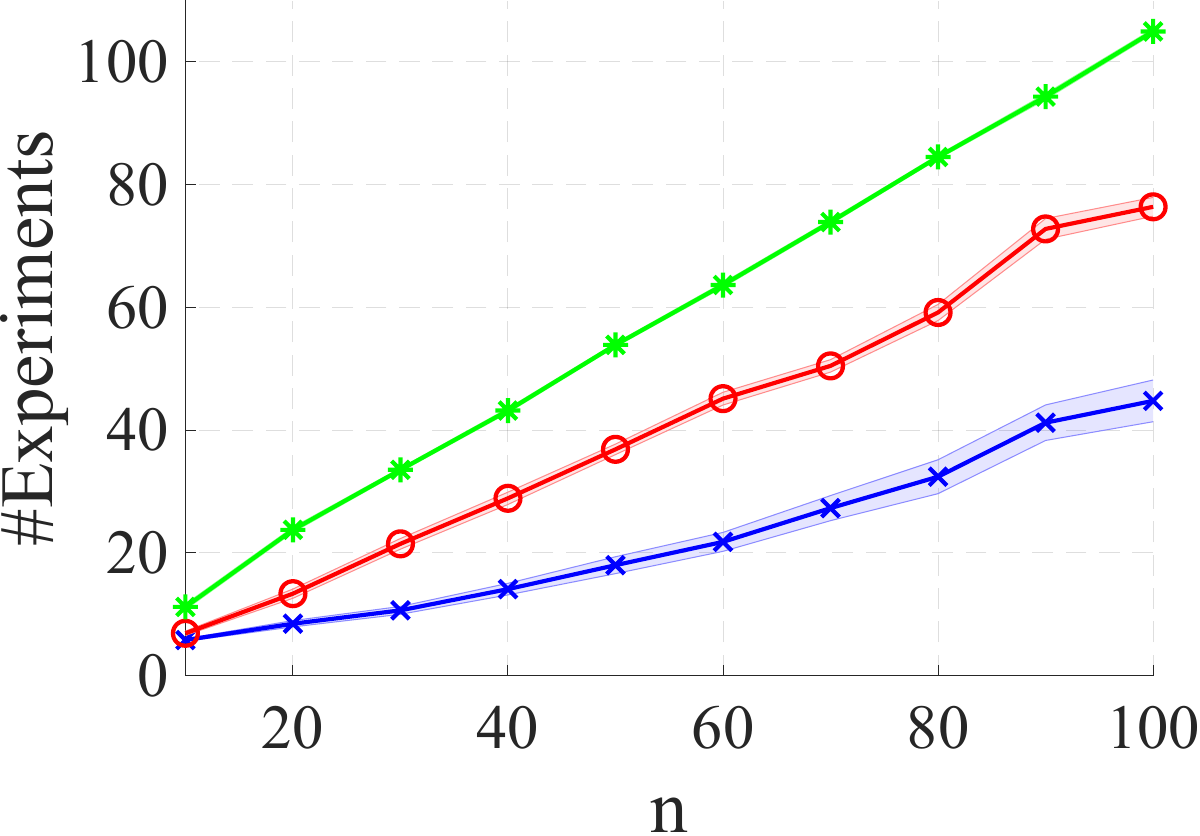}
                \caption{$p = \frac{\log(n)}{n}$, \#samples $= 200n$.}
                \label{fig: sim1}
            \end{subfigure}
            
            \begin{subfigure}{.6\textwidth}
                \centering
                \includegraphics[width=\textwidth]{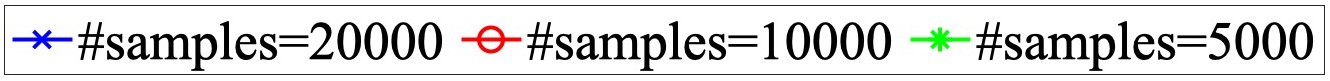}
            \end{subfigure}
            
            \begin{subfigure}{\textwidth}
                \centering
                \includegraphics[width=0.32\textwidth]{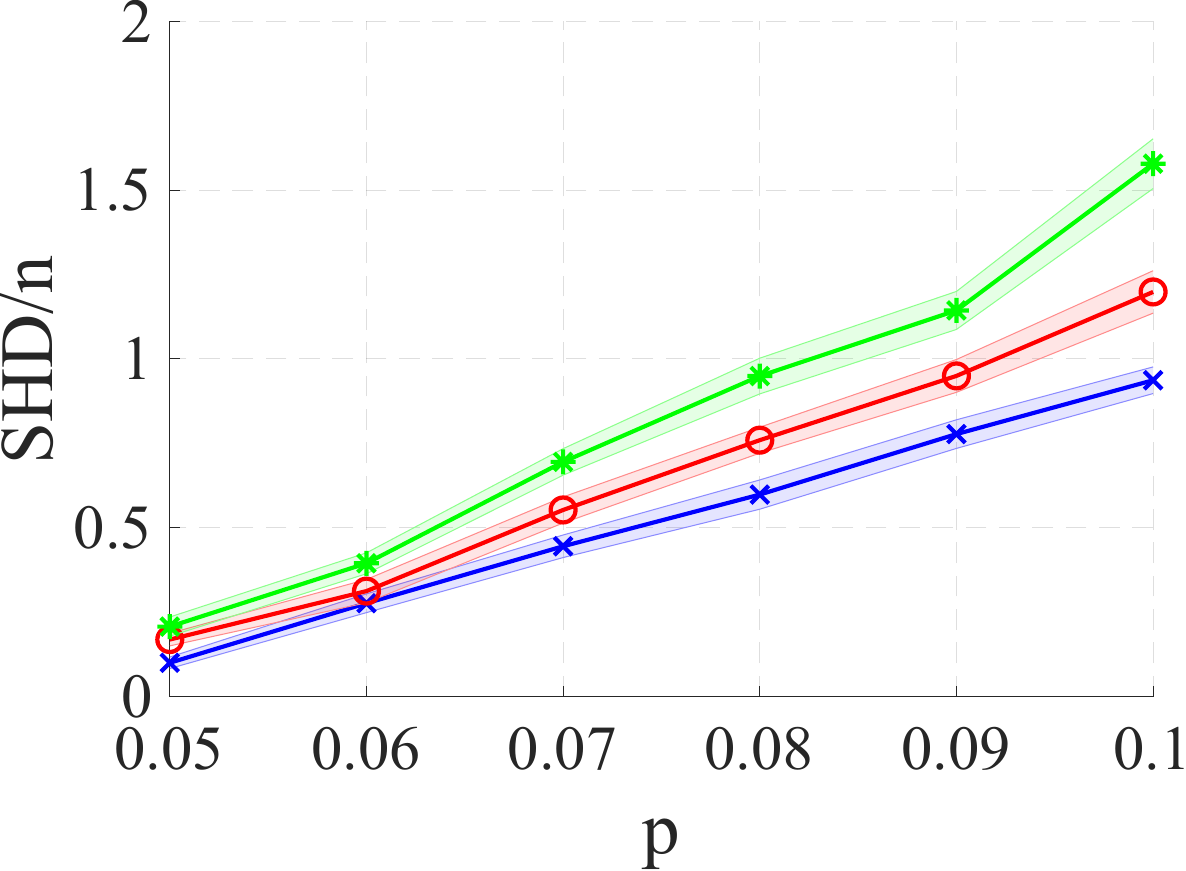}\hfill
                \includegraphics[width=0.32\textwidth]{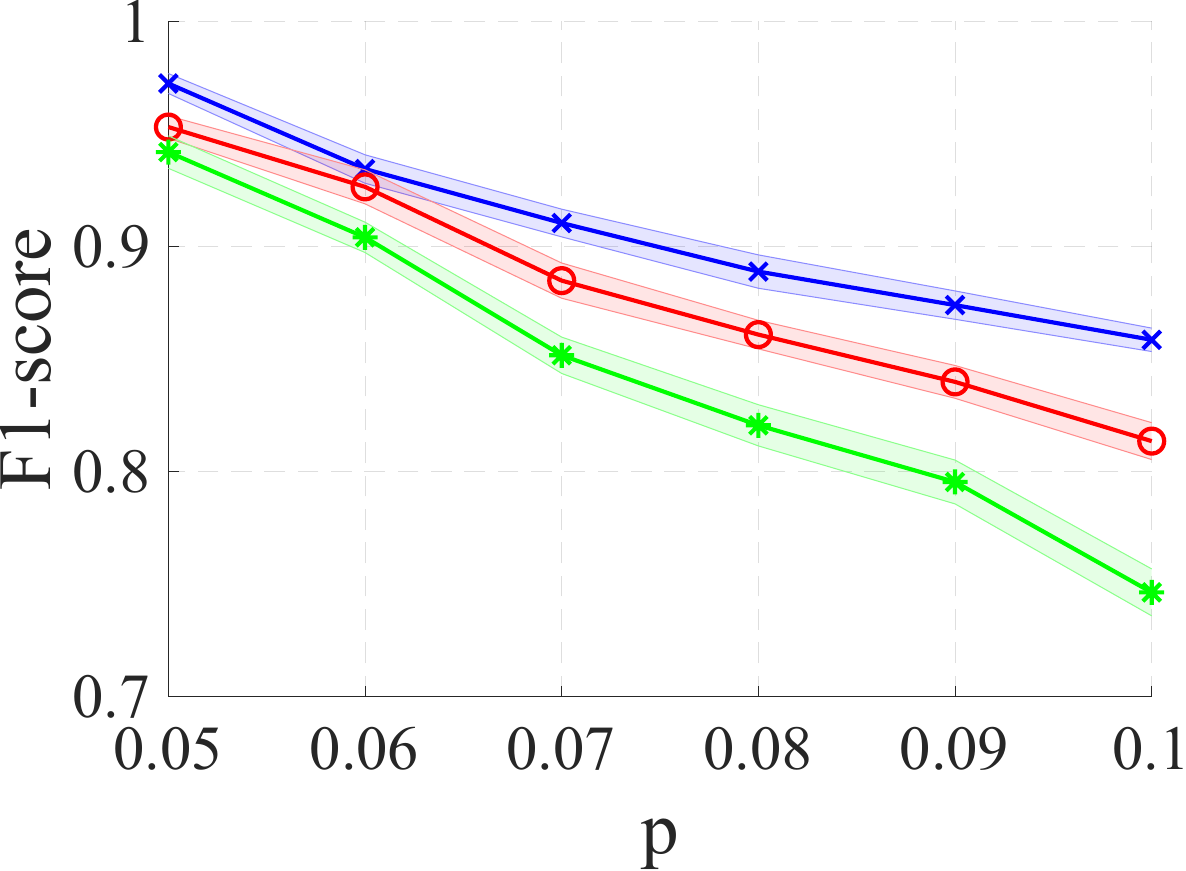}\hfill
                \includegraphics[width=0.32\textwidth]{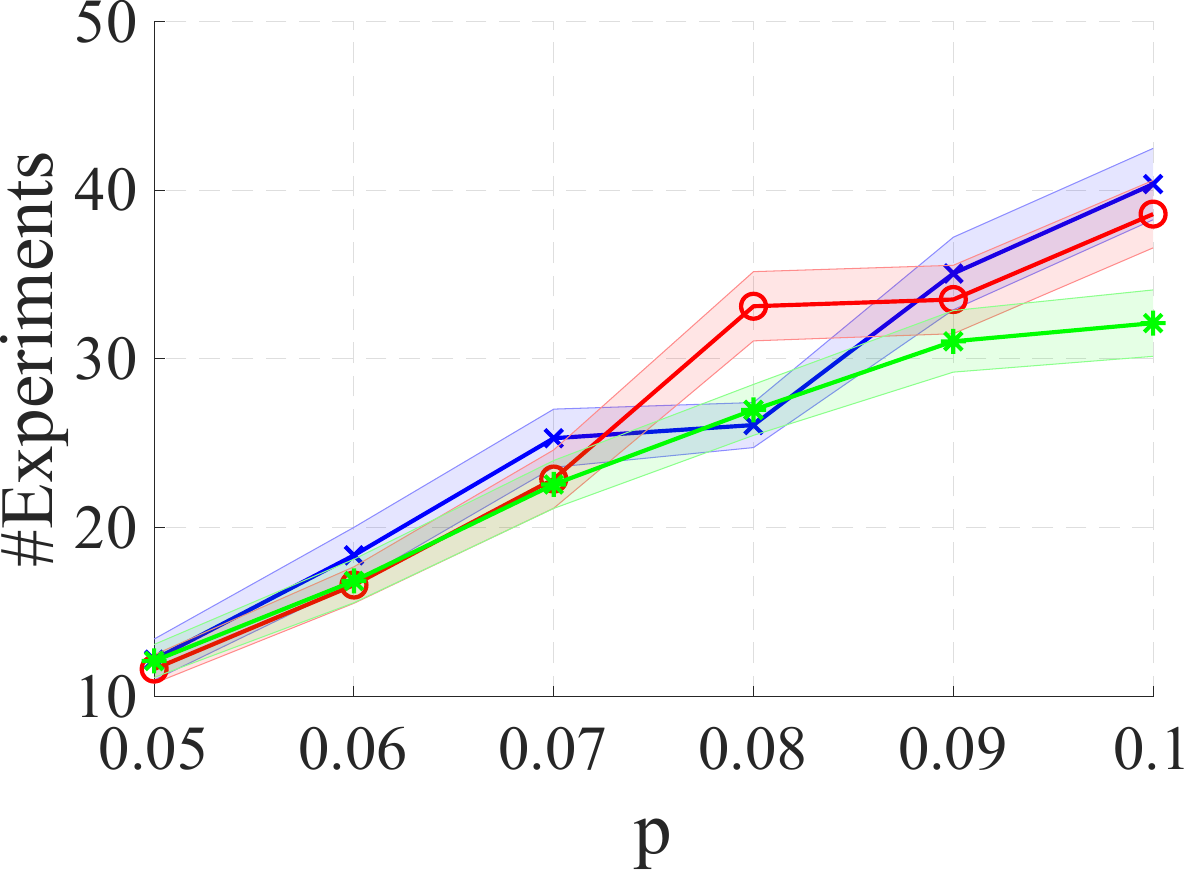}
                \caption{$n=50$, $b=25$.}
                \label{fig: sim2}
            \end{subfigure}
            \caption{Performance of our approach on random graphs generated from SBM($n,p,b$).}
            \label{fig: simulations}
        \end{figure}
        
        Figure \ref{fig: sim1} illustrates the effect of $n$ (number of vertices) and $b$ (the parameter that controls $\smax{\G}$) when $p = \frac{\log(n)}{n}$ (graph density) and the number of samples was fixed at $200n$.
        As can be seen, for moderate values of $b$, and accordingly $\smax{\G}$, the proposed algorithm achieves good accuracy in terms of F1-score and SHD.
        Moreover, the number of experiments scales linearly with $b$, which is consistent with our analysis.
        
        In Figure \ref{fig: sim2}, the effect of graph density is studied by varying $p$ for three different sample sizes (5000, 10000, 20000) when $n=50$ and $b=25$.
        This shows that once we reach an adequate number of samples, the sample size has a negligible effect on the number of experiments.
        Furthermore, we observe that the proposed approach performs better (both in terms of SHD and F1-score) on sparser graphs.
    
    \subsection{Evaluation Metrics} \label{subsec: metrics}
        We measured the accuracy of our algorithm by two commonly used metrics in the literature: F1-score and normalized Structural Hamming Distance (SHD/$n$).
        Herein and similar to \cite{mokhtarian2023novel}, we define these measures.
        
        Let $\G_1$ and $\G_2$ denote the true DG and the learned DG, respectively.
        We first define a few notations.
        True-positive (TP) is the number of edges that appear in both $\G_1$ and $\G_2$.
        False-positive (FP) is the number of edges that appear in $\G_2$ but do not exist in $\G_1$.
        False-negative (FN) is the number of edges in $\G_1$ that the algorithm failed to learn in $\G_2$.
        In this case, SHD is defined as follows.
        \begin{equation*}
            \text{SHD = FP + FN}, \quad \text{SHD/}n = \frac{\text{FP + FN}}{n}.
        \end{equation*}
        SHD is a non-negative integer, and smaller numbers indicate better accuracy.
        F1-score is defined by \textit{precision} and \textit{recall} in the following.
        \begin{equation*}
            \text{Precision = } \frac{\text{TP}}{\text{TP + FP}}, \quad \text{Recall = } \frac{\text{TP}}{\text{TP + FN}}, \quad 
            \text{F1-score = } 2 \times \frac{\text{Precision }\times \text{ Recall} }{\text{Precision }+ \text{ Recall}}.
        \end{equation*}
        Note that $0 \leq $ F1-score $\leq 1$ and larger numbers indicate better accuracy.

\section{Related Work} \label{sec: related work}
    The goal of causal discovery is to learn the causal graph of a system, which represents the existence and direction of relations among the variables of the system under study.
    In general, the causal graph can only be identified up to the Markov equivalence class (MEC) from mere observational data.
    \cite{richardson1996polynomial} provided necessary and sufficient conditions for the Markov equivalence of two DGs, based on which he proposed a consistent structure learning algorithm that can learn a DG up to the MEC \citep{richardson1996discovery}.
    Subsequently, \cite{mooij2020constraint} showed that the Fast Causal Inference (FCI) algorithm, originally designed for learning DAGs, can also learn a cyclic DG up to the MEC.
    \cite{forre2018constraint} introduced $\sigma$-connection graphs ($\sigma$-CG), a new class of mixed graphs (containing undirected, bidirected, and directed edges).
    % with additional structure.
    They proposed a causal discovery algorithm for $\sigma$-CGs, handling non-linear causal mechanisms, latent confounders, and data from multiple interventional distributions.
    \cite{ghassami2020characterizing} instead focused on the notion of distribution equivalence.
    They provided necessary and sufficient conditions for the distribution equivalence of two DGs for linear Gaussian causal DG models and proposed a score-based method for learning the structure from observational data.
    \cite{lacerda2008discovering} focused on the case of linear models with non-Gaussian noises and generalized the ICA-based approach of \cite{shimizu2006linear} to allow for cycles.
    
    As we discussed, to uniquely identify the causal graph, the gold standard is to perform experiments, leading to the experiment design problem.
    To the best of our knowledge, there is no previous work on the problem of experiment design in cyclic models.
    In the following, we mainly review previous work in acyclic models, where it has been studied extensively \citep{eberhardt2007causation, eberhardt2005number, eberhardt2008almost, he2008active, shanmugam2015learning}.
    We review the previous work based on the following three aspects of the experiment design problem.
    
    \begin{itemize}[leftmargin=*]
        \item \textbf{The objective of the problem:}
        The work on experiment design can be divided into two main categories.
        In the first category, the goal is to minimize the cost of experiments while it is required to learn the whole graph.
        This problem is referred to as the \emph{min-cost identification} problem.
        The second category aims to minimize the ambiguity about the causal graph while a limited budget for performing experiments is available.
        This problem is referred to as \emph{fixed budget} or \emph{budgeted experiment design}.

        \item \textbf{Adaptive versus non-adaptive methods:}
        An alternative way to divide methods is in terms of whether the interventions are performed adaptively or non-adaptively.
        \emph{Adaptive} methods sequentially perform experiments, where they exploit the results of previously performed experiments to design the latter ones.
        These methods are practical in cases where the experiments are not highly time-consuming.
        On the other hand, \emph{non-adaptive} methods design all the experiments simultaneously and perform them in parallel.
        
        \item \textbf{Bounded-size experiments:}
        In several applications, it is not feasible to perform large-size experiments.
        In such cases, the size of the designed experiments must be bounded by a given constant.
        This problem is referred to as the \emph{Bounded-size experiment design} problem.
    \end{itemize}
    The majority of earlier work focused on the min-cost identification problem in acyclic models.
    In particular, \cite{eberhardt2005number} proposed worst-case bounds on the number of required experiments where the number of intervened variables could be as large as half of the size of the graph.
    \cite{he2008active} proposed adaptive and non-adaptive algorithms for the case where the experiments are singleton, i.e., each experiment is comprised of a single variable.
    Their non-adaptive approach is brute force, and it can find the optimal solution.
    However, it is not scalable to large graphs as they enumerate all the DAGs in a MEC, and the number of DAGs in a MEC can grow super-exponentially with the number of variables.
    In the adaptive case, they presented a heuristic algorithm based on Shannon’s entropy to select the intervened variable in each step.
    \cite{hauser2014two} proposed an optimal algorithm for minimizing the number of undirected edges in the worst case when we are allowed to perform just one intervention.
    They further utilized this algorithm to propose a heuristic adaptive experiment design method.
    \cite{shanmugam2015learning} proposed a lower bound on the number of experiments for the adaptive methods based on the notion of separating systems.
    \cite{kocaoglu2017experimental} proposed a stage-wise algorithm for the experiment design problem in the presence of unobserved variables.
    First, the induced subgraph between observed variables is recovered, and then, by performing some “do-see” tests, the existence and the location of latent variables are identified.

    The experiment design problem has also been studied when intervention on each variable has a particular cost.
    In this setting, \cite{kocaoglu2017cost} proposed an optimal algorithm when there is no constraint on the number of interventions in each experiment. 
    \cite{greenewald2019sample} presented a 2-approximation adaptive algorithm for the tree causal structures.
    In a follow-up, \cite{squires2020active} proposed an adaptive algorithm for a more general class of causal graphs, matching the optimal number of interventions up to a multiplicative logarithmic factor.

    \cite{ghassami2018budgeted} introduced the fixed budget formulation of the experiment design problem.
    They considered the average number of recovered edges (after an intervention) as the objective function and showed that a general greedy algorithm is an approximation algorithm.
    Moreover, to estimate the objective function, they proposed a sampler from MEC, which evaluates the objective function by a Monte Carlo scheme.
    \cite{ghassami2019counting} presented a uniform sampler on clique trees for accelerating the generating of random DAGs from a given MEC.
    Then, they utilized it as a sub-routine for designing experiments.
    \cite{ghassami2019interventional} proposed an efficient exact algorithm for tree causal structures to minimize the number of undirected edges after performing interventions in the worst-case scenario.
    Later, \cite{ahmaditeshnizi2020lazyiter} proposed a method for iterating over all possible DAGs in the corresponding MEC after intervening on a variable, and introduced an exact algorithm for the fixed budget problem.
    The methods described above can be further reinforced by using state-of-the-art techniques for counting and sampling Markov equivalent DAGs.
    In particular, \cite{wienobst2021polynomial, wienobst2023polynomial} show that these tasks can be performed in polynomial time.

    The experiment design problem has also been studied in the Bayesian framework.
    For instance, \cite{agrawal2019abcd} proposed a tractable adaptive algorithm for the fixed budget problem with an approximation guarantee on sub-modularity.
    \cite{tigas2022interventions} proposed an adaptive experiment design method that designs not only the experiments but also the value at which each intervened variable should be set.

\section{Conclusion and Future Work} \label{sec: discussion}
    Feedback cycles in causal graphs are more the norm rather than the exception.
    We showed that in cyclic models, observational data is far less informative for structure learning, and it is necessary to solve the experiment design problem.
    The presence of cycles also introduces major challenges for the experiment design.
    For instance, intervening on a variable may not lead to recovering the presence or the direction of the edges incident to it.
    
    In this work, we proposed a unified experiment design framework that allows learning cyclic and acyclic graphs. 
    We further provided a theoretical analysis to calculate the required number and size of experiments in the worst case.
    The analysis demonstrated that our proposed approach is order-optimal in terms of the number of experiments up to an additive logarithmic term and optimal in terms of the size of the largest experiment required for unique identification of the causal graph in the worst case.

    In the following, we discuss potential future work.
    \begin{itemize}[leftmargin=*]
        \item 
            The main assumption of our proposed method is causal sufficiency.
            An important unsolved research problem is to relax this assumption and allow for latent confounders.
            We note that in the presence of latent confounders and even in acyclic models, experiment design is a challenging problem.
        \item 
            Although we assumed that the generative model is a simple SCM, for the soundness of our results, we only required that the interventional distribution exists (not necessarily unique) and that the CI assertions in the observational and interventional distributions are equivalent to either $d$-separation or $\sigma$-separations in the causal graph.
            Accordingly, another direction of future work is to characterize the class of SCMs satisfying the aforementioned assumptions.
        \item 
            In Scenario 2, we considered the $\sigma$-faithfulness assumption to ensure that any CI in the distribution implies $\sigma$-separation in the causal graph.
            As mentioned in Remark \ref{rem: d-sigma}, $\sigma$-faithfulness assumption is stronger than $d$-faithfulness assumption.
            It could be interesting to investigate how restrictive the assumption of $\sigma$-faithfulness is.
        \item
            We assumed that intervening in a variable removes its incoming edges.
            This type of intervention is commonly called \emph{hard intervention} (aka, perfect intervention).
            However, there are other types of interventions, such as \emph{soft-interventions}, in which the incoming edges will not necessarily be omitted (even in some cases, new edges will be added to the causal graph).
            Studying the problem of experiment design and investigating the required number and size of experiments under other types of interventions remains open.
        \item
            The lower bounds presented in Theorems \ref{thm: lower bound - size} and \ref{thm: lower bound - number} are \textit{worst-case} lower bounds, in the sense that for any constant $c$, there exists a DG $\G$ with $\smax{\G}=c$, which requires at least the specified number or size of interventions to be identified uniquely.
            A few recent works such as \cite{choo2022verification} have established \textit{instance-wise} lower bounds for experiment design on DAGs.
            These bounds can be used to develop instance-wise competitive guarantees for experiment design algorithms.
            The development of such bounds for cyclic causal models remains an important open problem.
    \end{itemize}

\acks{
    This research was in part supported by the Swiss National Science Foundation under NCCR Automation, grant agreement 51NF40\_180545 and Swiss SNF project 200021\_204355 /1.
}    

\bibliography{biblio}

\clearpage
\end{document}